\documentclass[10pt,journal,compsoc]{IEEEtran}
\ifCLASSOPTIONcompsoc
  \usepackage[nocompress]{cite}
\else
  \usepackage{cite}
\fi
\ifCLASSINFOpdf
  \usepackage[pdftex]{graphicx}
\else
\fi
\usepackage{amsmath}
\usepackage{amsfonts}
\usepackage{amsthm}
\usepackage{amssymb}
\usepackage{enumitem}
\usepackage[dvipsnames,svgnames,x11names,hyperref]{xcolor}
\usepackage{csquotes}
\usepackage[flushleft]{threeparttable}

\usepackage{algorithm,algorithmic}

\usepackage{multirow}
\usepackage{cite}

\newtheorem{theorem}{Theorem}
\newtheorem{corollary}{Corollary}
\newtheorem{lemma}{Lemma}
\newtheorem{prop}{Proposition}

\begin{document}
%

\title{A Simple and Fast Algorithm \\ for L1-norm Kernel PCA}

%
%
%
%

\author{Cheolmin~Kim, 
        Diego~Klabjan
\IEEEcompsocitemizethanks{\IEEEcompsocthanksitem Cheolmin Kim and Diego Klabjan are with the Department of Industrial Engineering and Management Science, Northwestern University, Evanston, IL, 60208.\protect\\
E-mail: cheolmkim@u.northwestern.edu, d-klabjan@northwestern.edu}}

\IEEEtitleabstractindextext{%
\begin{abstract}
We present an algorithm for L1-norm kernel PCA and provide a convergence analysis for it. While an optimal solution of L2-norm kernel PCA can be obtained through matrix decomposition, finding that of L1-norm kernel PCA is not trivial due to its non-convexity and non-smoothness. We provide a novel reformulation through which an equivalent, geometrically interpretable problem is obtained. Based on the geometric interpretation of the reformulated problem, we present a \enquote{fixed-point} type algorithm that iteratively computes a binary weight for each observation. As the algorithm requires only inner products of data vectors, it is computationally efficient and the kernel trick is applicable. In the convergence analysis, we show that the algorithm converges to a local optimal solution in a finite number of steps. Moreover, we provide a rate of convergence analysis, which has been never done for any L1-norm PCA algorithm, proving that the sequence of objective values converges at a linear rate. In numerical experiments, we show that the algorithm is robust in the presence of entry-wise perturbations and computationally scalable, especially in a large-scale setting. Lastly, we introduce an application to outlier detection where the model based on the proposed algorithm outperforms the benchmark algorithms.
\end{abstract}

}

\maketitle

\IEEEdisplaynontitleabstractindextext

%
\IEEEpeerreviewmaketitle

\IEEEraisesectionheading{\section{Introduction}\label{sec:introduction}}

\IEEEPARstart{P}{rincipal} Component Analysis (PCA) is one of the most popular dimensionality reduction techniques \cite{jolliffe2002principal}. Given a large set of possibly correlated features, it attempts to find a small set of features (\textit{principal components}) that retain as much information as possible. To generate such new dimensions, it linearly transforms original features by multiplying \textit{loading vectors} in a way that newly generated features are orthogonal and have the largest variance.

In traditional PCA, variance is measured using the $L_2$-norm. This has a nice property in that although the problem itself is non-convex, an optimal solution can be easily found through matrix factorization. With this property and easy interpretability, PCA is extensively used in a variety of applications. Nonetheless, it still has some limitations. First, since it generates a new dimension through a linear combination of features, it cannot capture non-linear relationships among features. Second, as it uses the $L_2$-norm for measuring variance, its outcome tends to be affected by influential outliers. In order to overcome these limitations, the following two approaches have been proposed.

{\bf Kernel PCA} The idea of kernel PCA is to map original features into a high-dimensional feature space, and perform PCA in that high-dimensional feature space \cite{scholkopf1997kernel}. Using a non-linear mapping, it can capture non-linear relationships among features in an efficient way using the \textit{kernel trick}. Using the trick, principal components can be computed with no explicit mapping.

{\bf$L_1$-norm PCA} To alleviate the effects of influential outliers, $L_1$-norm PCA uses the $L_1$-norm instead of the $L_2$-norm to measure variance. The $L_1$-norm is more advantageous than the $L_2$-norm in presence of observations having large feature values since it is less influenced by them. Using this property, more robust results can be obtained by $L_1$-norm PCA in the presence of influential outliers.

In this paper, we combine the two approaches for the variance maximization version of $L_1$-norm PCA. In what follows, we always refer to the variance maximization version of $L_1$-norm PCA which is not the same as minimizing reconstruction error with respect to the $L_1$-norm. Compared to $L_2$-norm kernel PCA, the kernel version of $L_1$-norm PCA is a hard problem in that it is not only non-convex but also non-smooth. However, through a novel  reformulation, we convert it to a geometrically interpretable problem where the objective is to minimize the $L_2$-norm of a vector subject to a linear constraint consisting of terms involving the $L_1$-norm. For the reformulated problem, we present a \enquote{fixed point} type algorithm that iteratively computes a weight of $-1$ or $1$ for each observation using the kernel matrix and previous weights. We show that the kernel trick is applicable to this algorithm. Moreover, we prove that the algorithm converges to a local optimal solution in a finite number of steps and the sequence of objective values converges at a linear rate. In numerical experiments, we computationally investigate the robustness of the algorithm and introduce an application to outlier detection. We also provide a runtime comparison to other robust kernel PCA algorithms and $L_2$-norm kernel PCA.

Our work has the following contributions. 
\begin{enumerate}
\item[1.] We provide a novel reformulation of $L_1$-norm kernel PCA and present an iterative algorithm based on the geometric interpretation of the reformulated problem. This approach is not specific to $L_1$-norm kernel PCA but can be applied to a more general problem. Particularly, its application to $L_2$-norm PCA results in Power iteration \cite{golub2012matrix}.

\item[2.] We not only prove convergence but also provide a rate of convergence analysis. Although many algorithms have been proposed for $L_1$-norm PCA, none of them provided a rate of convergence analysis. We stress that our analysis is for the kernel version which clearly covers $L_1$-norm PCA. Through a novel analysis, we show that the algorithm attains a linear rate of convergence.

\item[3.] We introduce a methodology based on $L_1$-norm kernel PCA for outlier detection and demonstrate that it outperforms the benchmark algorithms.
\end{enumerate}

The paper is organized as follows. Section \ref{literature review} reviews related works and points out how our work is different. Section \ref{reformulation} introduces a novel reformulation of $L_1$-norm kernel PCA and provides a geometric interpretation behind it. Based on the geometric interpretation, we present an iterative algorithm in Section \ref{algorithm}. Section \ref{convergence} provides a convergence analysis for it and the experimental results are followed in Section \ref{experiment}.

\section{Related Work} \label{literature review}
Extracting a low-rank representation from a large matrix is an important problem in machine learning and statistics. In a variety of contexts, many previous works \cite{candes2009exact,candes2011robust,xu2010robust,liu2013robust} have been proposed to address this problem. Recovering a low-rank matrix from a sampling of its entries is studied in \cite{candes2009exact}. Given that the number of sampled entries is sufficiently large, exact recovery is guaranteed with high probability by solving a simple convex optimization problem \cite{candes2009exact}. Assuming that a data matrix can be decomposed into the sum of a low-rank matrix $L_0$ and a sparse matrix $S_0$, a convex program (known as \textit{robust PCA}) that minimizes a weighted combination of the nuclear norm of $L_0$ and the $L_1$ norm of $S_0$ is presented in \cite{candes2011robust}. Also, a variant of robust PCA that identifies outliers by additionally imposing a column-sparse structure on $S_0$ is considered in \cite{xu2010robust}. Under some mild conditions, exact recovery is shown for both models \cite{candes2011robust,xu2010robust}. Moreover, exact recovery of mixture data is studied in \cite{liu2013robust,liu2014recovery,liu2016deterministic,liu2017blessing}. Utilizing a dictionary matrix, low-rank representation (LRR) \cite{liu2013robust} is shown to better handle mixture data than robust PCA. While matrix recovery is the main focus of theses works, our work considers dimensionality reduction with emphasis on robustness, especially focusing on kernel PCA with the $L_1$-norm.

To reduce the number of features in a robust way, the $L_1$-norm has been involved in many PCA studies \cite{brooks2013pure,Park2015,park2016iteratively,park2018three,nie2011robust,mccoy2011two,markopoulos2014optimal} and subspace estimation formulations \cite{ke2005,ding2006r}. Finding a subspace onto which the $L_1$ projections of data vectors have the smallest reconstruction error is studied in \cite{brooks2013pure}. Based on the observation that the $L_1$ projection occurs along a single unit direction, it finds an optimal subspace for each unit direction by solving $d$ least absolute deviation regression problems, each having one dimension as a dependent variable while having the other dimensions as independent variables. Using linear programming, this approach can find a global optimal subspace in polynomial time \cite{brooks2013pure}. 

Minimizing reconstruction error with respect to the $L_1$-norm is considered in \cite{park2016iteratively,park2018three,ke2005}. While the PCA problem of minimizing $\| M - XX^TM\|_1$ subject to $X^TX = I$ is considered in \cite{park2016iteratively}, the subspace estimation problem of minimizing $E(U,V)=\| M - UV\|_1$ is studied in \cite{ke2005} where $M$ is a data matrix. In order to solve the former problem, an iterative algorithm that computes a weight for each observation and applies $L_2$-norm PCA on the weighted data matrix is presented in \cite{park2016iteratively}. On the other hand, the latter problem is solved using alternative convex minimization based on the observation that $E(U,V)$ becomes a convex function once $U$ or $V$ is known. It alternatively optimizes one matrix at a time while keeping the other one fixed, repeating this process until convergence. Also, a subspace estimation formulation that minimizes reconstruction error with respect to the $R_1$-norm, $\|M-UV\|_{R_1}=\sum_{i=1}^n \| \text{x}_i - U\text{v}_i\|_2$ where $\text{x}_i$ is the $i^{th}$ column of $M$ and $\text{v}_i$ is that of $V$, is presented in \cite{ding2006r}. Since this formulation minimizes the sum of distances with respect to the $L_2$-norm, it is different from $L_2$-norm PCA which minimizes the sum of squared distances with respect to the $L_2$-norm. Nonetheless, they share the same property that they have a unique global solution which is rotational invariant \cite{ding2006r}.

Maximizing variance with respect to the $L_1$-norm, which we refer to as $L_1$-norm PCA, is studied in \cite{markopoulos2014optimal,mccoy2011two,Park2015,nie2011robust}. Our work also considers this formulation rather than the previous two since it has a favorable structure in that an optimal solution can be represented as a linear combination of data vectors with a weight of $-1$ or $1$. $L_1$-norm PCA is shown to be NP-hard in \cite{markopoulos2014optimal} and \cite{mccoy2011two}. Nevertheless, an algorithm finding a global optimal solution is proposed in \cite{markopoulos2014optimal}. Utilizing the auxiliary-unit-vector technique \cite{karystinos2010efficient}, it computes a global optimal solution with complexity $\mathcal{O}(n^{pr+p-1})$ where $n$ is the number of observations, $r$ is the rank of the data matrix, and $p$ is the desired number of principal components. Assuming $r$ and $p$ are fixed, the runtime of this algorithm is polynomial in $n$. However, if $n,p,r$ are large, it can be computationally prohibitive. Instead of finding a global optimal solution which is intractable in general, our work focuses on developing an efficient algorithm finding a local optimal solution for $L_1$-norm kernel PCA.

Recognizing the hardness of $L_1$-norm PCA, an approximation algorithm is presented in \cite{mccoy2011two} based on the known Nesterov's theorem \cite{nesterov1998semidefinite}. In this work, $L_1$-norm PCA is relaxed to a semi-definite programming (SDP) problem and alternatively, the SDP relaxation is considered. After solving the relaxed problem, it generates a random vector and uses randomized rounding to produce a feasible solution. This randomized algorithm is a $\sqrt{2/\pi}$-approximate algorithm in expectation. To achieve this approximation ratio with high probability, it performs randomized rounding multiple times and takes the one having the best objective value. Rather than providing an approximation guarantee by solving a relaxed problem, our work directly considers the kernel version of $L_1$-norm PCA and develops an efficient algorithm finding a local optimal solution.

Another approach utilizing a known mathematical programming model is introduced in \cite{Park2015} where the author proposes an iterative algorithm that solves a mixed integer programming problem in each iteration. Given an orthonormal matrix of loading vectors, it perturbs the matrix slightly in a way that the resulting matrix yields the largest objective value. After the perturbation, it uses singular value decomposition to recover orthogonality. The algorithm is completely different from the one proposed herein and the sequence of objective values does not necessarily improve over iterations. Unlike it, our algorithm guarantees that the sequence of objective values keeps improving and converges at a linear rate.

A simple numerical algorithm finding a local optimal solution is proposed in \cite{kwak2008principal}. In this work, an optimal solution is assumed to have a certain form, and weights involved in that form are updated in each iteration, improving the objective value. A similar algorithm and its extended version that finds multiple loading vectors at once are derived in \cite{nie2011robust} utilizing an optimization algorithm for general $L_1$-norm maximization problems. In the case of linear kernel, our algorithm uses the same framework as the one in \cite{kwak2008principal} and \cite{nie2011robust}. However, while the algorithm in \cite{kwak2008principal} is derived without any justification, we provide a geometric interpretation behind the algorithm, which is different from the derivation in  \cite{nie2011robust}. Moreover, we provide a rate of convergence analysis and introduce a kernel version, which are not considered in \cite{kwak2008principal} and \cite{nie2011robust}.

On other hand, the kernel version of $L_1$-norm PCA has been rarely studied. Due to the difficulty of applying the kernel trick to $L_1$-norm kernel PCA, an alternative method named \textit{nonlinear projection trick} is applied in \cite{kwak2013nonlinear}. Based on the finding that an optimal loading vector lies in the span of $\Phi(A)^T U \Lambda^{-1/2}$ where $\Phi(A)$ is a high-dimensionally mapped data matrix and $U\Lambda U^T$ is the eigenvalue decomposition of the kernel matrix $K$, it alternatively considers $L_1$-norm PCA having $U\Lambda^{1/2}$ in place of $\Phi(A)$ and solves it using the algorithm in \cite{kwak2008principal}. Another kernel extension of $L_1$-norm PCA is studied in \cite{xiao2013l1}. In this work, a linear system involving a kernel matrix is solved in each iteration and the resulting solution is used to update the iterate. While the algorithms in \cite{kwak2013nonlinear} and \cite{xiao2013l1} entail either eigenvalue decomposition or solving a linear system, our algorithm requires only a matrix-vector multiplication in each iteration, making it suitable in a large-scale setting.

\section{Kernel-based $L_1$-norm PCA Formulations} \label{reformulation}
We consider $L_1$-norm PCA in a high-dimensional feature space $F$.
Suppose we map data vectors $a_i \in \mathbb{R}^{d}$, $i=1,\ldots,n$ into a feature space $F$ by a possibly non-linear mapping $\Phi:\mathbb{R}^{d} \rightarrow F$.
Assuming that each feature is standardized with a mean of 0 and standard deviation of 1 and that the kernel matrix $K$ defined by $K_{ij}=\Phi(a_i)^T\Phi(a_j)$ satisfies
\begin{enumerate}
    \item $K_{ii} > 0$ for $1 \leq i \leq n$
    \item $|K_{ij}| < \infty$ for $1 \leq i,j \leq n$,
\end{enumerate}
the kernel version of $L_1$-norm PCA is formulated as
\begin{equation} \label{eq:original}
\begin{aligned}
& \underset{\text{x} \in F} {\text{maximize}}
& & f(\text{x})=\sum\limits_{i=1}^n |\Phi(a_i)^T\text{x}| \\
& \text{subject to}
& & \|\text{x}\|_2 = 1.
\end{aligned}
\end{equation}

This formulation having $\Phi(a_i)$ in place of $a_i$ extends the variance maximization version of $L_1$-norm PCA in the obvious way and is also considered in \cite{kwak2013nonlinear,xiao2013l1}. In this formulation, we only consider extracting the first loading vector. This assumption is justifiable since the subsequent loading vectors can be found by repeatedly solving \eqref{eq:original}. 
For example, once we obtain the first loading vector $\text{x}^*$, we can find the second loading vector by solving \eqref{eq:original} with $\Phi(a_i) - \text{x}^* (\Phi(a_i)^T \text{x}^*)$ in place of $\Phi(a_i)$.

Solving \eqref{eq:original} is not trivial since it has a convex non-smooth objective function to maximize and a Euclidean unit ball constraint. In order to better understand the problem and set an algorithmic foundation, we reformulate \eqref{eq:original} as

\begin{equation} \label{eq:new}
\begin{aligned}
& \underset{\text{x} \in F} {\text{minimize}}
& & g(\text{x})=\|\text{x}\|_2 \\
& \text{subject to}
& & \sum\limits_{i=1}^n |\Phi(a_i)^T\text{x}| = 1.
\end{aligned}
\end{equation}

In order to prove the equivalence of \eqref{eq:original} and \eqref{eq:new}, we argue that an optimal solution of one formulation can be derived from an optimal solution of the other formulation by means of some mapping. Two optimization problems are equivalent if there exists some mapping $h$ such that if $\text{x}^*$ is an optimal solution to one problem, then $h(\text{w}^*)$ is an optimal solution to the other problem, and vice versa for a possible different mapping function \cite{boyd2004convex}.

\begin{prop}
Let $\textup{x}_1^*$ and $\textup{y}_2^*$ be an optimal solution to \eqref{eq:original} and \eqref{eq:new}, respectively. Then, 
\begin{align*}
\textup{x}_2^*= \frac{\textup{x}_1^*}{\sum_{i=1}^n |\Phi(a_i)^T\textup{x}_1^*|}
\end{align*}
is an optimal solution to \eqref{eq:new}, and 
\begin{align*}
\textup{y}_1^* = \frac{\textup{y}_2^*}{\|\textup{y}_2^*\|_2}
\end{align*}
is an optimal solution to \eqref{eq:original}.
\end{prop}

\begin{proof}
It is easy to check that $\textup{x}_2^*$ is a feasible solution to \eqref{eq:new}. Suppose that $\textup{x}_2^*$ is not optimal to \eqref{eq:new}. Then, there exists some feasible $\text{z}$ such that 
\begin{align*}
\|\text{z}\|_2 < \|\text{x}_2^*\|_2.
\end{align*}
As $\text{z}$ is feasible to \eqref{eq:new}, we have
\begin{align*}
\sum\limits_{i=1}^n |\Phi(a_i)^T\text{z}| = 1.
\end{align*}
Let $\text{w}=\frac{\text{z}}{\|\text{z}\|_2}$. Then, we have
\begin{equation*}
f(\text{w})=\sum_{i=1}^n |\Phi(a_i)^T\text{w}|= \frac{\sum_{i=1}^n |\Phi(a_i)^T\text{z}|}{\|\text{z}\|_2} = \frac{1}{\|\text{z}\|_2}.
\end{equation*}
In the same way, we obtain
\begin{align*}
f(\text{x}_1^*)= \frac{1}{\|\text{x}_2^*\|_2}
\end{align*}
since
\begin{align*}
\text{x}_1^* = \frac{\text{x}_1^*}{\| \text{x}_1^* \|_2} = \frac{\sum_{i=1}^n |\Phi(a_i)^T\textup{x}_1^*|}{\sum_{i=1}^n |\Phi(a_i)^T\textup{x}_1^*|} \frac{\text{x}_1^*}{\| \text{x}_1^* \|_2} = \frac{\text{x}_2^*}{\| \text{x}_2^*\|_2}.
\end{align*}
This leads to
\begin{align*}
f(\text{x}_1^*) < f(\text{w}), 
\end{align*}
which contradicts the assumption that $\text{x}_1^*$ is an optimal solution of \eqref{eq:original}. Therefore, $\textup{x}_2^*$ is optimal to \eqref{eq:new}

On the other hand, it is obvious that $\text{y}_1^*$ is feasible to \eqref{eq:original}. To derive a contradiction, suppose that $\text{y}_1^*$ is not optimal to \eqref{eq:original}. Then, there exists some feasible  $\text{w}$ such that 
\begin{align*}
\sum\limits_{i=1}^n |\Phi(a_i)^T \text{y}_1^*|<\sum\limits_{i=1}^n |\Phi(a_i)^T\text{w}|. 
\end{align*}
Let
\begin{align*}
\text{z}= \frac{\text{w}}{\sum_{i=1}^n |\Phi(a_i)^T\text{w}|}.
\end{align*}
Then, we have 
\begin{align*}
g(\text{z}) = \frac{\|\text{w}\|_2}{\sum_{i=1}^n |\Phi(a_i)^T\text{w}|} = \frac{1}{\sum_{i=1}^n |\Phi(a_i)^T\text{w}|}
\end{align*}
since $\|\text{w}\|_2 = 1$.
In the same way, we obtain
\begin{align*}
g(\text{y}_2^*)= \frac{1}{\sum_{i=1}^n |\Phi(a_i)^T \textup{y}_1^*|}.
\end{align*}
for
\begin{align*}
\text{y}_2^*= \frac{\textup{y}_1^*}{\sum_{i=1}^n |\Phi(a_i)^T \textup{y}_1^*|} 
\end{align*}
due to $\|\textup{y}_1^*\|_2 = 1$.
As a result, we have
\begin{align*}
g(\text{y}_2^*) > g(\text{z}),
\end{align*}
contradicting the assumption that $\text{y}_2^*$ is optimal to \eqref{eq:new}. Therefore, $\text{y}_1^*$ is optimal to \eqref{eq:original}.
\end{proof}

To understand formulation \eqref{eq:new}, we first look at the constraint set, 
\begin{align*}
\partial P = \Big\{ \text{x} \big| \sum\limits_{i=1}^n |\Phi(a_i)^T\text{x}| = 1 \Big\}.
\end{align*}
Geometrically, this constraint set is symmetric with respect to the origin and represents the boundary of polytope
\begin{align*}
P=\Big\{ \text{x} \big| \sum\limits_{i=1}^n |\Phi(a_i)^T\text{x}| \leq 1 \Big\}.
\end{align*}
It is easy to check that $P$ is a polytope since it can be written as the intersection of a finite set of linear inequalities each having the form of $\sum_{i=1}^n c_i\Phi(a_i)^T\text{x} \leq 1$
where $c_i \in \{-1,1\}$.
As the objective function measures the distance from the origin, formulation \eqref{eq:new} can be understood as a problem of finding the closest point to the origin from the boundary of the polytope $\partial P$. The following proposition shows that an optimal solution $\text{x}^*$ must be perpendicular to one of the faces of $\partial P$. 

\begin{prop} \label{prop2}
An optimal solution $\textup{x}^*$ is perpendicular to the face which it lies on.
\end{prop}

\begin{proof}
Let $\text{x}^*$ be an optimal solution of \eqref{eq:new} and define a face $E$ such that
\begin{align*}
E = \Big \{ \text{x} \big | \sum_{i=1}^n c_i^* \Phi(a_i)^T \text{x}=1 \Big \} \cap \partial P
\end{align*}
where
\begin{align*}
c_i^* = \text{sgn}(\Phi(a_i)^T \text{x}^*)
\end{align*}
for $1 \leq i \leq n$. If $\text{x}^*$ is not perpendicular to face $E$, then 
\begin{align*}
\text{w}= \frac{\sum_{i=1}^n \Phi(a_i)c_i^*}{\|\sum_{i=1}^n \Phi(a_i)c_i^*\|_2^2}
\end{align*}
is the closest point to the origin from
\begin{align*}
\Big \{ \text{x} \big | \sum_{i=1}^n c_i^* \Phi(a_i)^T \text{x}=1 \Big \}
\end{align*}
having 
\begin{align}
\|\text{w}\|_2 < \|\text{x}^*\|_2.    
\label{eq:norm-w-norm-x-inequality}
\end{align}
Let
\begin{align*}
\text{z}= \frac{\text{w}}{\sum_{i=1}^n |\Phi(a_i)^T \text{w}|}. 
\end{align*}
Then, $\text{z}$ is feasible to \eqref{eq:new} and has the objective value of 
\begin{align}
\| \text{z} \|_2 = \frac{\|\text{w}\|_2}{\sum_{i=1}^n |\Phi(a_i)^T \text{w}|}.
\label{eq:g-z-objective}
\end{align}
From
\begin{align*}
\| \sum \limits_{i=1}^n \Phi(a_i)c_i^* \|_2^2 = \sum \limits_{i=1}^n \Phi(a_i)^Tc_i^*(\sum \limits_{j=1}^n \Phi(a_j)c_j^*),
\end{align*}
we have
\begin{align*}
& \sum \limits_{i=1}^n |\Phi(a_i)^T (\sum \limits_{j=1}^n \Phi(a_j)c_j^*)| - \| \sum \limits_{i=1}^n \Phi(a_i)c_i^* \|_2^2 \geq 0
\end{align*}
resulting in
\begin{align}
\sum \limits_{i=1}^n |\Phi(a_i)^T \text{w}| = \frac{\sum_{i=1}^n |\Phi(a_i)^T (\sum_{j=1}^n \Phi(a_j)c_j^*)|}{\| \sum_{i=1}^n \Phi(a_i)c_i^* \|_2^2} \geq 1.
\label{eq:f-w-lower-bound-1}
\end{align}
As a result, by \eqref{eq:norm-w-norm-x-inequality}, \eqref{eq:g-z-objective}, and \eqref{eq:f-w-lower-bound-1}, we have
\begin{align*}
\| \text{z} \|_2 \leq \|\text{w}\|_2 < \|\text{x}^*\|_2,
\end{align*}
which contradicts the assumption that $\text{x}^*$ is optimal to \eqref{eq:new}. Therefore, $\text{x}^*$ must be perpendicular to $E$.
\end{proof}
Proposition \ref{prop2} is important since it helps to characterize the form of an optimal solution $\text{x}^*$. From Proposition \ref{prop2}, we obtain the following corollary.
\begin{corollary} \label{corollary1}
An optimal solution $\textup{x}^*$ of \eqref{eq:new} has the form of 
\begin{align*}
\textup{x}^*= \frac{\textup{y}^*}{\sum_{i=1}^n |\Phi(a_i)^T \textup{y}^*|}
\end{align*}
for some $\textup{y}^*$ and $\textup{c}^*$ such that
\begin{align*}
\textup{y}^* = \sum \limits_{i=1}^n \Phi(a_i)c_i^*
\end{align*}
and 
\begin{align*}
c_i^* = \textup{sgn}(\Phi(a_i)^T \textup{y}^*),
\end{align*}
for $1 \leq i \leq n$.
\end{corollary}
The characterization of an optimal loading vector using a sign vector is first proposed in \cite{kwak2008principal} without any justification. However, we provide a derivation based on the geometry of $\partial P$, which is different from the one in \cite{nie2011robust} that uses the KKT conditions.
Moreover, since we have
\begin{align}
\|\text{x}^*\|_2 = \frac{\|\text{y}^*\|_2}{\sum_{i=1}^n |\Phi(a_i)^T \text{y}^*|} = \frac{1}{\|\sum_{i=1}^n \Phi(a_i)c_i^*\|_2} \label{eq:conversion}
\end{align}
due to
\begin{align*}
\sum_{i=1}^n |\Phi(a_i)^T \text{y}^*| &= \sum_{i=1}^n c_i^*\Phi(a_i)^T \text{y}^* = \|\sum_{i=1}^n \Phi(a_i)c_i^*\|_2^2,
\end{align*}
we can further show that an optimal solution of formulation \eqref{eq:new} can be found from an optimal solution of the following binary problem,
\begin{equation} \label{eq:binary}
\begin{aligned}
& \underset{\text{c} \in {\{-1,1\}}^n} {\text{maximize}}
& & \|\sum_{i=1}^n \Phi(a_i)c_i\|_2^2.
\end{aligned}
\end{equation}

\begin{prop} \label{prop4}
Let $\textup{c}^*$ be an optimal solution of binary formulation \eqref{eq:binary}. Then, 
\begin{align*}
\textup{y}^* = \sum \limits_{i=1}^n \Phi(a_i)c_i^*
\end{align*}
satisfies 
\begin{align}
c_i^* = \textup{sgn}(\Phi(a_i)^T \textup{y}^*),
\label{eq:sign-normal-stability}
\end{align}
for $1 \leq i \leq n$. Moreover,
\begin{align*}
\textup{x}^*= \frac{\textup{y}^*}{\sum_{i=1}^n |\Phi(a_i)^T \textup{y}^*|}
\end{align*}
is an optimal solution of formulation \eqref{eq:new}.
\end{prop}
\begin{proof}
To deduce a contradiction, let us assume that there exists some nonempty set $J \subset \{1,\ldots,n\}$ such that 
\begin{align*}
c_j^* = -\text{sgn}(\Phi(a_j)^T \text{y}^*)
\end{align*}
for $j \in J$. Since $\text{c}^*$ is an optimal solution of \eqref{eq:binary}, flipping the sign of $c_j^*$ for $j \in J$ must not improve the objective value of \eqref{eq:binary}. However, for any $j \in J$, flipping the sign of $c_j^*$ results in
\begin{align*}
\|\text{y} - 2\Phi(a_j)c_j^*\|_2^2  > \|\text{y}\|_2^2
\end{align*}
since
\begin{align*}
\| \text{y}^* - 2\Phi(a_j)c_j^* \|_2^2 = \| \text{y} \|_2^2 +4|\text{y}^T(\Phi(a_j))| + 4\|\Phi(a_j)\|_2^2
\end{align*}
and $\|\Phi(a_j)\|_2^2 > 0$. This contradicts the assumption that $\text{c}^*$ is an optimal solution to \eqref{eq:binary}. Therefore, $\text{y}^*$ must satisfy 
\begin{align*}
c_i^* = \text{sgn}(\Phi(a_i)^T \text{y}^*)
\end{align*}
for $1 \leq i \leq n$. Since $\text{y}^*$ and $\text{c}^*$ satisfy \eqref{eq:sign-normal-stability} and  $\text{c}^*$ maximizes the objective value of \eqref{eq:binary},
\begin{align*}
\text{x}^*= \frac{ \text{y}^*}{\sum_{i=1}^n |\Phi(a_i)^T \text{y}^*|}
\end{align*}
is a minimizer of \eqref{eq:new} due to Corollary \ref{corollary1} and \eqref{eq:conversion}.
\end{proof}
The following result has been shown in \cite{markopoulos2014optimal} for the linear kernel case but here we generalize it.
\begin{corollary}
Formulation \eqref{eq:new} is equivalent to formulation \eqref{eq:binary}.
\end{corollary}
\begin{proof}
Based on Corollary \ref{corollary1} and \eqref{eq:conversion}, we can formulate \eqref{eq:new} as
\begin{align*}
& \underset{\text{c} \in {\{-1,1\}}^{n}} {\text{maximize}} 
& & \|\sum \limits_{i=1}^n \Phi(a_i)c_i\|_2^2 \\
& \text{subject to}
& & \text{y} = \sum \limits_{i=1}^n \Phi(a_i)c_i \\
& & & c_i = \text{sgn}(\Phi(a_i)^T \text{y}), \quad 1 \leq i \leq n.
\end{align*}
Since an optimal solution $\text{c}^*$ to \eqref{eq:binary} satisfies the constraints of the above optimization problem by Proposition \ref{prop4}, the two formulations are essentially the same.
\end{proof}

It is interesting to note that we can reduce formulation \eqref{eq:binary} to the weighted max-cut problem since
\begin{align}
\|\sum \limits_{i=1}^n \Phi(a_i)c_i\|_2^2 &= \sum \limits_{i,j=1}^n K_{ij} + \sum \limits_{i,j=1}^n (-2 K_{ij}) \Big( \frac{1-c_i c_j}{2} \Big). \label{eq:binary-expansion}
\end{align}
Using the above reduction, we can alternatively consider the weighted max-cut problem on a complete graph with weight $w_{ij}=-K_{ij}$.
Therefore, a popular approximation algorithm for the weighted max-cut problem \cite{goemans1995improved} can be used to solve \eqref{eq:binary}. However, due to the additional constant terms in \eqref{eq:binary-expansion}, this does not imply a constant worst case approximation ratio algorithm for \eqref{eq:binary}.

\section{Algorithm} \label{algorithm}
In this section, we develop an algorithm that finds a local optimal solution to \eqref{eq:new} based on the findings in Section \ref{reformulation}. Before giving details of the algorithm, we first provide the idea behind the algorithm.

The main idea of the algorithm is to move along the boundary of $P$ so that the $L_2$-norm of $\text{x}_k$ successively decreases. Figure \ref{idea} illustrates a step of the algorithm. Starting with an iterate $\text{x}^{k}$, we first identify the hyperplane $\text{h}^{k}$ which the current iterate $\text{x}^{k}$ lies on. After identifying the equation of $\text{h}^{k}$, we find the closest point to the origin from $\text{h}^{k}$, which we denote by $\text{z}^{k}$. After that, we obtain $\text{x}^{k+1}$ by projecting $\text{z}^{k}$ to the constraint set $\partial P$, which is done by multiplying an appropriate scalar between 0 and 1. We repeat this process until the sequence of iterates $\{\text{x}^{k}\}$ converges.

\begin{figure}[ht]
\vskip 0.2in
\begin{center}
\centerline{\includegraphics[scale=0.6]{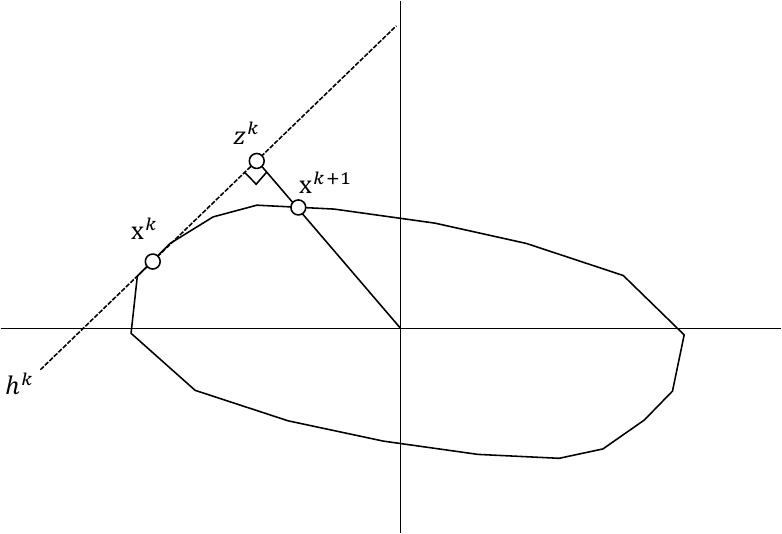}}
\caption{Geometric derivation of the algorithm}
\label{idea}
\end{center}
\vskip -0.2in
\end{figure} 

Now, we develop an algorithm based on the above idea. 
Given $\text{x}^k$, we define the normal vector $\text{y}^k$ at $\text{x}^k$ by
\begin{align}
\text{y}^k = \sum_{i=1}^n \Phi(a_i)c_i^k \label{eq:def_y}
\end{align}
using the sign vector $\text{c}^k=[c_1^k,\ldots,c_n^k]^T$ defined by
\begin{align*}
    c_i^k = \text{sgn}(\Phi(a_i)^T \text{x}^k)
\end{align*}
for $1 \leq i \leq n$. Using the normal vector $\text{y}^k$ at $\text{x}^k$, we can find the equation of hyperplane $\text{h}^k$ as
\begin{align}
(\text{y}^k)^T(\text{x}-\text{x}^k)=0. \label{eq:h^k}
\end{align}
The closest point $\text{z}^k$ to the origin from $\text{h}^k$ has the form of 
\begin{align}
\text{z}^k=s\text{y}^k. \label{eq:z-00}    
\end{align}
Plugging \eqref{eq:z-00} into \eqref{eq:h^k}, we have 
\begin{align*}
s= \frac{(\text{y}^k)^T \text{x}^k}{(\text{y}^k)^T \text{y}^k}
\end{align*}
resulting in
\begin{align}
\text{z}^k = \frac{(\text{y}^k)^T\text{x}^k}{(\text{y}^k)^T\text{y}^k} \text{y}^k. \label{eq:z-0}
\end{align}
Projecting $\text{z}^k$ to $\partial P$, we obtain
\begin{align}
\text{x}^{k+1} = \frac{\text{z}^k}{\sum_{i=1}^n |\Phi(a_i)^T \text{z}^k|}. \label{eq:def_x-z}
\end{align}

Using
\begin{align}
(\text{y}^k)^T\text{x}^k = \sum \limits_{i=1}^n \Phi(a_i)^T\text{x}^kc_i^k = \sum \limits_{i=1}^n |\Phi(a_i)^T\text{x}^k|=1, \label{eq:z}
\end{align}
we can further write \eqref{eq:z-0} as
\begin{align}
\text{z}^{k} &= \frac{\text{y}^k}{\|\text{y}^k\|_2^2} \label{eq:def_z}
\end{align}
leading to
\begin{align}
\text{x}^{k+1} = \frac{\text{y}^k}{\sum_{i=1}^n |\Phi(a_i)^T \text{y}^k|}. \label{eq:def_x-y}
\end{align}
Also, from \eqref{eq:def_y} and 
\begin{align*}
{\sum_{i=1}^n |\Phi(a_i)^T \text{y}^k|} = {\sum_{i=1}^n \Phi(a_i)^T \text{y}^k c_i^k} = (\text{c}^k)^TK\text{c}^k,
\end{align*}
we can represent $\text{x}^{k+1}$ as a function of $\text{c}^k$ by
\begin{align}
\text{x}^{k+1} = \frac{\sum_{i=1}^n \Phi(a_i)c_i^k}{(\text{c}^k)^TK\text{c}^k}. \label{eq:def_x_2}
\end{align}
Since
\begin{align*}
c_i^{k+1} & = \text{sgn}(\Phi(a_i)^T\text{x}^{k+1}) = \text{sgn} ( K_{i \cdot} \text{c}^k),
\end{align*}
we can update $c_i^{k+1}$ using only $K$ and $\text{c}^k$ by
\begin{align*}
\text{c}^{k+1} = \text{sgn}(K\text{c}^k).
\end{align*}
Moreover, from
\begin{align*}
\|\text{x}^{k+1}-\text{x}^{k}\|_2^2 = \frac{(\text{c}^k-\text{c}^{k+1})^TK(\text{c}^k-\text{c}^{k+1})}{(\text{c}^k)^TK\text{c}^k (\text{c}^{k+1})^TK\text{c}^{k+1}},
\end{align*}
we can represent the termination criteria $\text{x}^{k+1} = \text{x}^{k}$ by
\begin{align*}
(\text{c}^k-\text{c}^{k+1})^TK(\text{c}^k-\text{c}^{k+1}) = 0.
\end{align*}

On the other hand, due to non-convexity of the problem, the algorithm can be stuck at a local optimum unless it is initialized close to a global optimum. In order to obtain a good initial iterate $\text{x}^0$, we consider each $\Phi(a_j)$ and select the one such that $\Phi(a_j)/\|\Phi(a_j)\|_2$ yields the largest objective value for $f$, which is computed by
\begin{align}
\frac{\Sigma_{i=1}^n | \Phi(a_i)^T\Phi(a_j) |}{\|\Phi(a_j)\|_2} = \frac{\Sigma_{i=1}^n | K_{ij}|}{\sqrt{K_{jj}}}.
\label{eq:initialization-objective-value}
\end{align}
Once we find the index $j^*$ maximizing \eqref{eq:initialization-objective-value}, we set
\begin{align*}
\text{x}^0 = \frac{\Phi(a_{j^*})}{\sum_{i=1}^n |\Phi(a_i)^T \Phi(a_{j^*})|}
\end{align*}
resulting in
\begin{align*}
c_i^0 = \text{sgn}(\Phi(a_i)^T\text{x}^0) = \text{sgn}(\Phi(a_i)^T\Phi(a_{j^*})) = \text{sgn}(K_{ij^*}).
\end{align*}
Since an optimal loading vector $\text{x}^*$ must be located somewhere between $\Phi(a_i)$ where $1 \leq i \leq n$, the above initialization scheme is likely to yield an initial iterate $\text{x}^0$ close to the optimal loading vector $\text{x}^*$.

Summarizing all the above, we obtain Algorithm \ref{alg:L1KernelPCA}.
\vspace{-1mm}
\begin{algorithm}[ht]
   \caption{\textit{L\textsubscript{1}}-norm Kernel PCA}
   \label{alg:L1KernelPCA}
   \begin{algorithmic}
   \STATE {\bfseries Input:} kernel matrix $K$ \\
   \STATE find $j^* = \text{arg max}_{1 \leq j \leq n} {\Sigma_{i=1}^n | K_{ij}|}/{\sqrt{K_{jj}}}$ \\
   \STATE initialize the sign vector $\text{c}^0$ with $\text{c}_i^0 = \text{sgn}(K_{ij^*})$ \\
   \STATE $k \leftarrow -1$
   \REPEAT
   \STATE $k \leftarrow k+1$
   \STATE compute $\text{c}^{k+1} = \text{sgn} (K \text{c}^k)$
   \UNTIL{$(\text{c}^k-\text{c}^{k+1})^TK(\text{c}^k-\text{c}^{k+1})=0$}
   \STATE {\bfseries Output:} sign vector $\text{c}^*$ \\
   \end{algorithmic}
\end{algorithm}
\vspace{-1mm}

Once we get the output $\text{c}^{*}$ from Algorithm \ref{alg:L1KernelPCA}, we can compute principal scores with no explicit mapping. For example, the principal component of the $i^\text{th}$ observation can be computed by
\begin{align*}
\frac{\Phi(a_i)^T\text{x}^*}{\|\text{x}^*\|_2} &= \frac{\sum_{j=1}^n \Phi(a_i)^T\Phi(a_j)c_j^{*}}{\sqrt{\sum_{i=1}^n \sum_{j=1}^n \Phi(a_i)^T\Phi(a_j)c_i^{*}c_j^{*}}}  \\
&= \frac{K_{i \cdot} \text{c}^*}{\sqrt{(\text{c}^{*})^TK\text{c}^{*}}}.
\end{align*}

Also, we can proceed to find more principal components with no explicit mapping. Noting that computing a loading vector and principal components requires only the kernel matrix, it suffices to update the kernel matrix each time a new loading vector is found. Fortunately, updating the kernel matrix can be done with no explicit mapping by
\begin{align*}
\widetilde{K}_{ij} &= \Bigg(\Phi(a_i) - \frac{{\Phi(a_i)}^T \text{x}^*}{\| \text{x}^* \|_2^2} \text{x}^* \Bigg)^T \Bigg(\Phi(a_j) - \frac{{\Phi(a_j)}^T \text{x}^*}{\| \text{x}^* \|_2^2} \text{x}^* \Bigg) \\
&= \Phi(a_i)^T\Phi(a_j) - \frac{{\Phi(a_i)}^T \text{x}^*{\Phi(a_j)}^T \text{x}^*}{\| \text{x}^*\|_2^2} \\
&= K_{ij} - \frac{K_{i\cdot} \text{c}^* K_{j\cdot} \text{c}^*}{(\text{c}^*)^TK\text{c}^*},
\end{align*}
which is equivalent to
\begin{align*}
\widetilde{K} = K - \frac{(K\text{c}^*)(K\text{c}^*)^T}{(\text{c}^*)^TK\text{c}^*}
\end{align*}
in a matrix form.

From $\text{y}_k = \nabla f(\text{x}_k)$,
update rule \eqref{eq:def_x-y} can be understood as projecting a gradient $\nabla f(\text{x}_k)$ to the constraint set $\partial P$ in each iteration. In this sense, Algorithm \ref{alg:L1KernelPCA} resembles Power iteration \cite{golub2012matrix} for solving the eigenvalue problem, and interestingly, the application of our framework to the eigenvalue problem yields the same algorithm. 
The framework developed in this work such as reformulation, geometric interpretation and algorithm derivation is not specific to $L_1$-norm kernel PCA but can be extended to solve a more general problem. For example, our approach can be used to solve
\begin{align*}
\text{maximize} \quad f(\text{x}) \quad \text{subject to} \quad \|\text{x}\|_2 = 1
\end{align*}
for any function $f$ that is scale-invariant (\textit{homogeneous} or \textit{homothetic}). The application of our framework to this problem yields the following update rule
\begin{align*}
\text{x}^{k+1} \leftarrow \nabla f(\text{x}^k)/{\| \nabla f(\text{x}^k) \|_2}.
\end{align*}

Compared to the other $L_1$-norm kernel PCA algorithms \cite{kwak2013nonlinear,xiao2013l1} considering the same formulation \eqref{eq:original}, Algorithm \ref{alg:L1KernelPCA} is much simple and computationally efficient as it involves just one matrix-vector multiplication in each iteration. In the case of L1-KPCA \cite{xiao2013l1}, a system of linear equations having the form of 
\begin{align*}
K \eta = \Sigma_{j=1}^n \text{c}_j^k K_{\cdot j} 
\end{align*}
is repeatedly solved. Solving the above linear system is not only computationally costly but also numerically unstable since it is singular due to the presence of non-trivial solution $\text{c}^k$.
On the other hand, KPCA-L1 \cite{kwak2013nonlinear} requires one matrix-vector multiplication but it does not directly consider the kernel matrix $K$. Instead, the eigenvalue decomposition of the kernel matrix $K=U\Lambda U^T$ must be computed before starting to find each loading vector. Also, $U\Lambda^{1/2}$ is involved in computation instead of the kernel matrix $K$. As Algorithm \ref{alg:L1KernelPCA} entails neither solving a linear system nor computing the eigenvalue decomposition of $K$, it is computationally more efficient than the other algorithms.

When it comes to initialization, L1-KPCA \cite{xiao2013l1} uses the optimal loading vector from $L_2$-norm kernel PCA. While KPCA-L1 \cite{kwak2013nonlinear} finds the data vector having the largest norm and uses its normalization for the initial iterate, Algorithm \ref{alg:L1KernelPCA} finds the normalized data vector with the largest objective value for $f$ and set it to be the initial iterate. As the initialization scheme of Algorithm \ref{alg:L1KernelPCA} is based on the objective fucntion $f$ while the others are not, it is more likely to obtain a good initial iterate compared to the others.

\section{Convergence Analysis} \label{convergence}
In this section, we provide a convergence analysis of Algorithm \ref{alg:L1KernelPCA}. We first prove that the algorithm converges in a finite number of iterations, and then provide a rate of convergence analysis. Before proving the finite convergence of the algorithm, we first show that the sequence $\{\|\text{x}^k\|_2\}$ generated by Algorithm \ref{alg:L1KernelPCA} is non-increasing.
\begin{lemma} \label{lemma1}
Let $\{\textup{x}_k\}$ and $\{\textup{z}_k\}$ be a sequence of vectors generated by Algorithm \ref{alg:L1KernelPCA} and \eqref{eq:def_z}, respectively. Then, we have 
\begin{align*}
\| \textup{x}^{k+1} \|_2 \leq \| \textup{z}^{k} \|_2 \leq \| \textup{x}^{k} \|_2.
\end{align*}
Moreover, if $\| \textup{x}^{k} \|_2 = \| \textup{z}^{k} \|_2$, we have $\text{x}^k = r \text{y}^k$ for some $r \in \mathbb{R}$.
\end{lemma}
\begin{proof}
The inequality $\| \text{z}^{k} \|_2 \leq \| \text{x}^{k} \|_2$ follows from
\begin{align}
\| \text{x}^{k} \|_2^2 - \| \text{z}^{k} \|_2^2 &= \| \text{x}^{k} \|_2^2 - \frac{1}{\| \text{y}^{k} \|_2^2} \nonumber\\
&= \| \text{x}^{k} \|_2^2 - \frac{((\text{y}^k)^T \text{x}^k)^2}{\| \text{y}^{k} \|_2^2} \nonumber \\
&=\frac{\| \text{x}^{k} \|_2^2\| \text{y}^{k} \|_2^2 - ((\text{y}^k)^T \text{x}^k)^2}{\| \text{y}^{k} \|_2^2} \nonumber \\
&\geq 0 \nonumber
\end{align}
where the second equality holds follows from \eqref{eq:z} and the last inequality holds due to the Cauchy-Schwarz inequality. If $\| \textup{x}^{k} \|_2 = \| \textup{z}^{k} \|_2$, the Cauchy-Schwarz inequality becomes an equality resulting in
\begin{align*}
\text{x}^k = r \text{y}^k    
\end{align*}
for some $r \in \mathbb{R}$.

Next, from \eqref{eq:def_x-z}, we have
\begin{align}
\|\text{x}^{k+1} \|_2^2 = \frac{\|\text{z}^k\|_2^2}{(\sum_{i=1}^n |\Phi(a_i)^T\text{z}^k|)^2}.
\label{eq:lemma1-2nd}
\end{align}
Using \eqref{eq:z-0}, we can represent the denominator as
\begin{align*}
\sum \limits_{i=1}^n |\Phi(a_i)^T\text{z}^k| &= \frac{\sum_{i=1}^n |\Phi(a_i)^T\text{y}^k|}{(\text{y}^k)^T \text{y}^k}.
\end{align*}
From
\begin{align*}
\sum_{i=1}^n |\Phi(a_i)^T\text{y}^k| &= \sum_{i=1}^n |\Phi(a_i)^T(\sum_{j=1}^n \Phi(a_j)c_j^k)| \\
&= \sum_{i=1}^n |\sum_{j=1}^n \Phi(a_i)^T\Phi(a_j)c_i^kc_j^k|
\end{align*}
and
\begin{align*}
(\text{y}^k)^T\text{y}^k = \sum_{i=1}^n \sum_{j=1}^n \Phi(a_i)^T\Phi(a_j)c_i^kc_j^k,
\end{align*}
we obtain
\begin{align}
\sum \limits_{i=1}^n |\Phi(a_i)^T\text{z}^k| &= \frac{\sum_{i=1}^n |\sum_{j=1}^n \Phi(a_i)^T\Phi(a_j)c_i^kc_j^k|}{\sum_{i=1}^n \sum_{j=1}^n \Phi(a_i)^T\Phi(a_j)c_i^kc_j^k} \label{eq:rate}
\end{align}
resulting in
\begin{align}
\sum \limits_{i=1}^n |\Phi(a_i)^T\text{z}^k| \geq 1. \label{eq:7}
\end{align}
By \eqref{eq:lemma1-2nd} and \eqref{eq:7}, we have
\begin{align*}
\|\text{x}^{k+1} \|_2^2 \leq \|\text{z}^k\|_2^2.
\end{align*}

\end{proof}
\begin{lemma} \label{lemma2}
If 
\begin{align*}
\| \textup{x}^{k} \|_2 = \| \textup{x}^{k+1} \|_2,
\end{align*}
then, we have 
\begin{align*}
\textup{x}^k =\frac{\textup{y}^k}{\|\textup{y}^k\|_2^2}, \quad \textup{y}^k =\frac{\textup{x}^k}{\|\textup{x}^k\|_2^2},
\end{align*}
resulting in 
\begin{align*}
\textup{x}^{k} = \textup{x}^{k+1}.
\end{align*}
\end{lemma}

\begin{proof}
Since $\| \textup{x}^{k} \|_2 = \| \textup{x}^{k+1} \|_2$, we have 
\begin{align*}
\| \text{z}^{k} \|_2 = \| \text{x}^{k} \|_2, \quad \text{x}^k=r\text{y}^k
\end{align*}
by Lemma \ref{lemma1} for some $r \in \mathbb{R}$. Using \eqref{eq:z}, we have
\begin{align*}
r=\frac{1}{\|\text{y}^k\|_2^2}
\end{align*}
 resulting in 
\begin{align*}
\textup{x}^k=\frac{\textup{y}^k}{\|\textup{y}^k\|_2^2}.
\end{align*}
In the same way, we can show 
\begin{align*}
\textup{y}^k=\frac{\textup{x}^k}{\|\textup{x}^k\|_2^2}.
\end{align*}
Since this implies $\text{z}^k=\text{x}^k$ by \eqref{eq:def_z}, we finally have
\begin{align*}
\text{x}^{k+1} = \frac{\text{z}^k }{\sum_{i=1}^n |\Phi(a_i)^T\text{z}^k|} 
= \frac{\text{x}^k }{\sum_{i=1}^n |\Phi(a_i)^T\text{x}^k|} 
= \text{x}^{k}
\end{align*}
where the first equality follows from \eqref{eq:def_x-z} and the last equality holds from the feasibility of $\text{x}^{k}$.
\end{proof}

\begin{theorem} \label{proof_convergence}
The sequence $\{\textup{x}^k\}$ converges in a finite number of steps.
\label{thm:finite-convergence}
\end{theorem}

\begin{proof}
Suppose the sequence $\{\text{x}^k\}$ does not converge.
As an iterate $\text{x}^k$ is solely determined by a sign vector $\text{c}^k \in \{-1,+1\}^n$, the number of possible vectors that $\text{x}^k$ can take is finite. Therefore, if the sequence $\{\text{x}^k\}$ does not converge, some vectors must appear more than once. Without loss of generality, let $\text{x}^l = \text{x}^{l+m}$. By Lemma \ref{lemma1}, we have
\begin{align*}
\| \text{x}^{l+m} \|_2 = \| \text{x}^l \|_2 \geq \| \text{x}^{l+1} \|_2 \geq ... \geq \| \text{x}^{l+m} \|_2
\end{align*}
forcing us to have 
\begin{align*}
\| \text{x}^l \|_2 = \| \text{x}^{l+1} \|_2 = ... = \| \text{x}^{l+m} \|_2. \label{eq:thm1-x-eq}
\end{align*}
This implies
\begin{align*}
\text{x}^l = \text{x}^{l+1} = ... = \text{x}^{l+m}
\end{align*}
by Lemma \ref{lemma2}, contradicting the assumption that the sequence $\{\text{x}^k\}$ does not converge. Therefore, the sequence $\{\text{x}^k\}$ generated by Algorithm \ref{alg:L1KernelPCA} must converge in a finite number of steps.
\end{proof}

Next, we show that the sequence of $\{\|\text{x}^k\|_2\}$ generated by Algorithm \ref{alg:L1KernelPCA} converges at a linear rate. Although Theorem \ref{thm:finite-convergence} shows that the algorithm converges in a finite number of steps, it may take an exponential number of steps to converge, due to the combinatorial structure of the problem, making it not appropriate in a large-scale setting. To make sure that this does not happen for Algorithm \ref{alg:L1KernelPCA}, we additionally prove linear convergence, which ensures that the optimality gap decreases no worse than a certain rate $\rho < 1$. Since this result implies that an $\epsilon$-optimal local solution can be attained after $\mathcal{O}(1/(1-\rho) \text{ log} (1/\epsilon))$ iterations, we can obtain a near-optimal solution after a sufficient number of iterations without waiting for an exponential number of steps.

\begin{theorem} \label{rate_convergence}
Let Algorithm \ref{alg:L1KernelPCA} start from $\textup{x}^0$ and terminate with $\textup{x}^*$ at iteration $k^*$. Then, for some $\rho < 1$, we have
\begin{align*}
\| \textup{x}^{k} \|_2 - \| \textup{x}^{*} \|_2 \leq \rho^{k} (\| \textup{x}^{0} \|_2 - \| \textup{x}^{*} \|_2)
\end{align*}
for $k < {k^*}$.
\end{theorem}

\begin{proof} From \eqref{eq:def_x-z}, we have 
\begin{align*}
\| \text{x}^{k} \|_2 = \frac{\|\text{z}^{k-1}\|_2}{\sum_{i=1}^n |\Phi(a_i)^T\text{z}^{k-1}|}.
\end{align*}
Since $\| \text{z}^{k-1} \|_2 \leq \| \text{x}^{k-1} \|_2$ holds by Lemma \ref{lemma1}, we obtain
\begin{equation} \label{eq:convieq1}
	\| \text{x}^{k} \|_2 \leq \frac{\|\text{x}^{k-1}\|_2}{\sum_{i=1}^n |\Phi(a_i)^T\text{z}^{k-1}|}. 
\end{equation}
Subtracting $\|\text{x}^*\|_2$ to \eqref{eq:convieq1}, we have
\begin{align}
\| \text{x}^{k} \|_2 - \| \text{x}^{*} \|_2 &\leq \frac{\|\text{x}^{k-1}\|_2}{\sum_{i=1}^n |\Phi(a_i)^T\text{z}^{k-1}|} - \| \text{x}^{*} \|_2 \nonumber \\
&\leq \frac{1}{\sum_{i=1}^n |\Phi(a_i)^T\text{z}^{k-1}|} (\| \text{x}^{k-1}\|_2 - \| \text{x}^{*} \|_2) \label{eq:convieq2}
\end{align}
where the last inequality follows from \eqref{eq:7}.

By induction on \eqref{eq:convieq2}, we obtain
\begin{align}
\| \text{x}^{k} \|_2 - \| \text{x}^{*}\|_2 \leq (\| \text{x}^0\|_2 - \| \text{x}^{*} \|_2) \prod_{l=1}^{k} \frac{1}{\sum_{i=1}^n |\Phi(a_i)^T\text{z}^{l-1}|}. \label{eq:linearRate}
\end{align}
From \eqref{eq:7}, we know that
\begin{align*}
\sum_{i=1}^n |\Phi(a_i)^T\text{z}^{l-1}| \geq  1.
\end{align*}
If
\begin{align*}
\sum_{i=1}^n |\Phi(a_i)^T\text{z}^{l-1}|=1,
\end{align*}
we have 
\begin{align*}
\frac{\sum_{i=1}^n |\sum_{j=1}^n \Phi(a_i)^T\Phi(a_j)c_i^{l-1}c_j^{l-1}|}{\sum_{i=1}^n \sum_{j=1}^n \Phi(a_i)^T\Phi(a_j)c_i^{l-1}c_j^{l-1}} = 1
\end{align*}
resulting in
\begin{align*}
c_i^{l-1} = \text{sgn} \Big( \sum_{j=1}^n \Phi(a_i)^T\Phi(a_j)c_j^{l-1} \Big).
\end{align*}
Since this implies
\begin{align*}
\text{c}^{l} = \text{sgn} (K \text{c}^{l-1}) = \text{c}^{l-1},
\end{align*}
we have
\begin{align*}
\text{x}^{l} = \text{x}^{l+1}.
\end{align*}
Therefore, as long as $l < k^*$, we must have 
\begin{align*}
\sum_{i=1}^n |\Phi(a_i)^T\text{z}^{j-1}| > 1.
\end{align*}

For $\text{c} \in \{-1,1\}^n$, let
\begin{align*}
\rho(\text{c}) = \frac{\sum_{i=1}^n \sum_{j=1}^n \Phi(a_i)^T\Phi(a_j)c_ic_j}{\sum_{i=1}^n |\sum_{j=1}^n \Phi(a_i)^T\Phi(a_j)c_ic_j|}
\end{align*}
and define
\begin{align*}
\rho = \text{max}_{\text{c}\in \{-1,1\}^n} \rho(\text{c}) \text{ subject to } \rho(\text{c})<1.
\end{align*}
Then, for $l<k^*$, we have
\begin{align*}
\frac{1}{\sum_{i=1}^n |\Phi(a_i)^T\text{z}^{j-1}|} = \rho(\text{c}^{j-1}) < \rho < 1.
\end{align*}
By combining it with \eqref{eq:linearRate}, we get the desired result.
\end{proof}
As shown in Theorem \ref{rate_convergence}, no matter where the algorithm starts, the sequence of objective values of \eqref{eq:new} converges at a linear rate. Now, we show that we can obtain a local optimal solution of \eqref{eq:original} by scaling the output of Algorithm \ref{alg:L1KernelPCA}.
\begin{theorem} \label{optimality}
Let the output of Algorithm \ref{alg:L1KernelPCA} be $\textup{x}^*$. Then, 
\begin{align*}
\bar{\textup{x}}^*= \frac{\textup{x}^*}{\|\textup{x}^*\|_2}
\end{align*}
is a local optimal solution of \eqref{eq:original}.
\end{theorem}

\begin{proof}
It is easy to see that $\bar{\textup{x}}^*$ is feasible. Since $\text{x}^*$ is the output of Algorithm \ref{alg:L1KernelPCA}, 
\begin{align*}
\text{y}^* = \frac{\text{x}^*}{\|\text{x}^*\|_2^2}
\end{align*}
holds by Lemma \ref{lemma2}. Next, consider 
\begin{align*}
L(\lambda,\text{x})= \sum\limits_{i=1}^n |\Phi(a_i)^T\text{x}| - \lambda(\|\text{x}\|_2^2-1).
\end{align*}
From 
\begin{align*}
\nabla_{\text{x}} L(\lambda,\text{x})=  \sum\limits_{i=1}^n \text{sgn}(\Phi(a_i)^T\text{x})\Phi(a_i)-2\lambda \text{x},
\end{align*}
we have
\begin{align*}
\nabla_{\text{x}} L(\lambda,\bar{\textup{x}}^*) &=  \sum\limits_{i=1}^n \text{sgn}(\Phi(a_i)^T\bar{\textup{x}}^*)\Phi(a_i)-2\lambda \bar{\textup{x}}^* \\
&= \sum\limits_{i=1}^n \text{sgn}(\Phi(a_i)^T\textup{x}^*)\Phi(a_i)-2\lambda \bar{\textup{x}}^*\\
&= \text{y}^*-2\lambda\bar{\textup{x}}^* \\
&= \frac{\text{x}^*}{\|\text{x}^*\|_2^2}-2\lambda\bar{\textup{x}}^* \\
&= \bigg( \frac{1}{\|\textup{x}^*\|_2}-2\lambda \bigg) \bar{\textup{x}}^*.
\end{align*}
Therefore, with 
\begin{align*}
\lambda^* = \frac{1}{2\|\textup{x}^*\|_2},
\end{align*}
we have 
\begin{align*}
\nabla_{\text{x}} L(\lambda^*,\bar{\textup{x}}^*)=0,
\end{align*}
meaning that $(\lambda^*,\bar{\textup{x}}^*)$ satisfies the first-order necessary conditions. Moreover, from 
\begin{align*}
\nabla_{\text{x}\text{x}} L(\lambda^*,\bar{\textup{x}}^*)=-2\lambda^*I \prec 0,
\end{align*}
the second-order sufficient condition is also satisfied. Since $(\lambda^*,\bar{\textup{x}}^*)$ satisfies the first and second order conditions, from the theory of constrained optimization, $\bar{\textup{x}}^*$ is a local optimal solution of \eqref{eq:original}. 
\end{proof}

\section{Experimental Results} \label{experiment}
In this section, we assess the robustness and scalability of Algorithm \ref{alg:L1KernelPCA} by running it on several tasks and compare it with other kernel PCA algorithms. 
First, we apply them on datasets having entry-wise perturbations and investigate how well each algorithm extracts principal components in a noisy setting. 
Next, we introduce their application to outlier detection and compare their performance with other popular outlier detection models.
Lastly, we provide their runtime comparison.

In addition to Algorithm \ref{alg:L1KernelPCA}, the two other $L_1$-norm kernel PCA algorithms (KPCA-L1\cite{kwak2013nonlinear}, L1-KPCA\cite{xiao2013l1}), the kernel version of $R_1$-norm PCA (R1-KPCA\cite{ding2006r}) and  $L_2$-norm kernel PCA (L2-KPCA\cite{scholkopf1997kernel}) are considered in the experiments. 
While $R_1$-norm PCA\cite{ding2006r} is not originally designed to incorporate kernels, we include it as it is easy to develop a kernel variant. Other $L_1$-norm PCA algorithms were also considered but since it is not straightforward to develop a kernel version for them, they are disregarded.

\subsection{Robust Extraction of PCs}
\label{subsec:robust_extraction}
To measure robustness, we first run the algorithms on datasets having entry-wise perturbations (noisy datasets) to obtain loading vectors. After that, we compute how much variation in the perturbation-excluded datasets (normal datasets) is explained by the loading vectors obtained from the noisy datasets. For this experiment, we prepare synthetic datasets having entry-wise perturbations so that loading vectors obtained by running $L_2$-norm kernel PCA on noisy and normal datasets are different from each other.

To generate synthetic datasets, we first construct a $1000 \times 50$ data matrix with the rank of 10 following the data generation procedure in \cite{park2016iteratively}. While the largest size in \cite{park2016iteratively} is $300 \times 50$, we choose the size of $1000 \times 50$ to consider larger datasets. To obtain entry-wise perturbations, we corrupt $r\%$ of observations by adding some random noises. We refer to the resulting dataset as a noisy dataset and the noisy dataset without the entry-wise perturbations as a normal dataset. For each value of $r \in \{5,10,15,20,25,30\}$, we generate 10 instances.

Let $K$ denote a kernel matrix of a normal dataset and $\text{x}_1,\ldots,\text{x}_p$ be $p$ loading vectors obtained by running $L_2$-norm kernel PCA on $K$. Also, let $\tilde K$ be a kernel matrix of a noisy dataset and $\tilde{\text{x}}_1,\ldots,\tilde{\text{x}}_p$ be loading vectors obtained by running one of the kernel PCA algorithms (Algorithm \ref{alg:L1KernelPCA}, KPCA-L1, L1-KPCA, R1-KPCA, L2-KPCA) on $\tilde K$. Assuming that the normal dataset is standardized,
\begin{align}
\sum\limits_{j=1}^p \sum\limits_{i=1}^n (\Phi(a_i)^T \tilde{\text{x}}_j)^2 = \sum\limits_{j=1}^p \tilde{\text{x}}_j^TK \tilde{\text{x}}_j \label{eq:explained_variance}
\end{align}
represents the amount of variation in the normal dataset explained by the $p$ loading vectors $\tilde{\text{x}}_1,\ldots,\tilde{\text{x}}_p$ where $n$ is the number of observations in the normal dataset. After dividing \eqref{eq:explained_variance} by $\sum_{j=1}^p {\text{x}}_j^TK {\text{x}}_j$, which is the maximum amount of variation in the normal dataset that the $p$ orthogonal vectors can explain, and multiplying by 100, we get the following measure:
\begin{align}
& \text{(Total Explained Variation)} & 100 \times \frac{\sum_{j=1}^p \tilde{\text{x}}_j^TK \tilde{\text{x}}_j}{\sum_{j=1}^p {\text{x}}_j^TK {\text{x}}_j}. \label{eq:explained_variance_scaled}
\end{align}

Metric \eqref{eq:explained_variance_scaled} captures how well the loading vectors obtained from the noisy dataset explain variation in the normal dataset with respect to the $L_2$-norm. Therefore, it can be used to measure the robustness of each kernel PCA algorithm in the presence of entry-wise perturbations. 
For example, if one algorithm has a value close to one, then it is robust with respect to entry-wise perturbations. 
Using this metric, we compare the robustness of Algorithm \ref{alg:L1KernelPCA} with that of KPCA-L1, L1-KPCA, R1-KPCA, and L2-KPCA. For each value of $r$, we compute \eqref{eq:explained_variance_scaled} for the ten datasets with $p=4$ and average them. We arbitrarily choose $p=4$ since the result is consistent regardless of the choice of $p$. Figure \ref{fig:robustness_linear} shows the results for the linear kernel and Figure \ref{fig:robustness} shows the results for the Gaussian kernel with the width parameter $\sigma$ varying from $10$ to $25$.

\begin{figure}[ht]
\centering
 \includegraphics[trim={1.5cm 0 1cm 0},clip,scale=0.35]{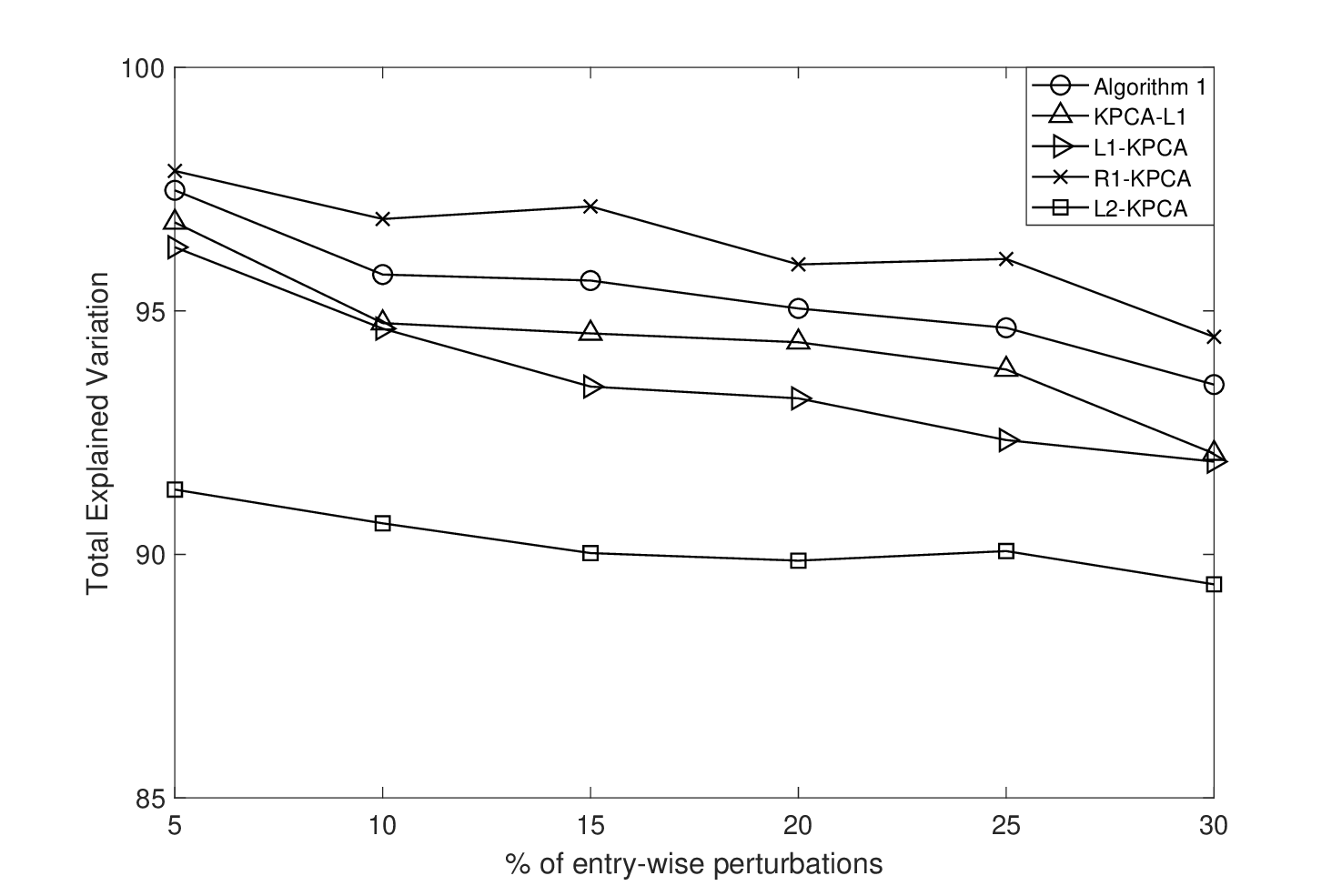}
\caption{Robust Extraction of PCs (Linear Kernel)}
\label{fig:robustness_linear}
\end{figure}

\begin{figure*}[ht]
\centering
\includegraphics[trim={3.2cm 0 3cm 0},clip,scale=0.55]{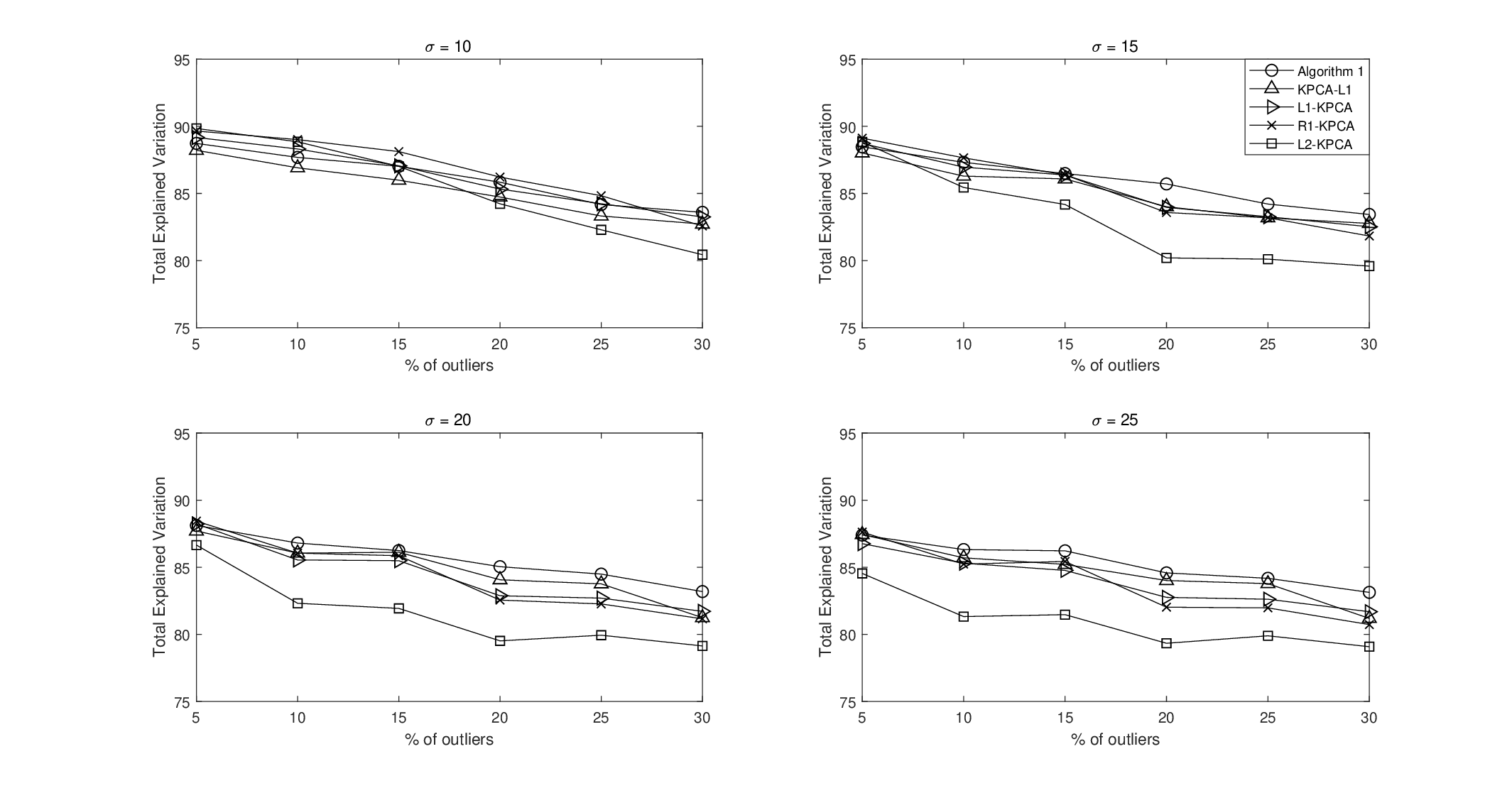}
\caption{Robust Extraction of PCs (Gaussian Kernel with $\sigma$ ranging from $10$ to $25$)}
\label{fig:robustness}
\end{figure*}

In the case of the linear kernel, R1-KPCA achieves the best performance for all values of $r$ followed by the $L_1$-norm based kernel PCA algorithms and L2-KPCA. 
While the loading vectors from L2-KPCA explain about $90\%$ of the variation, those from R1-KPCA, Algorithm \ref{alg:L1KernelPCA}, KPCA-L1, and L1-KPCA explain around $96\%$,$95\%$,$94\%$, and $93\%$ of the variation, respectively.
This demonstrates the robustness of the $R_1$-norm and $L_1$-norm based kernel PCA algorithms with respect to the presence of entry-wise perturbations.
Among the three $L_1$-norm based kernel PCA algorithms, Algorithm \ref{alg:L1KernelPCA} consistently outperforms KPCA-L1 and L1-KPCA by $1\%$ and $2\%$, respectively.
As the percentage of corrupted observations ($r\%$) increases, the total explained variation tends to decrease for all of them but the gaps between them remain the same.
 
When the Gaussian kernel is used, the results are slightly different depending on the value of $r$ and $\sigma$. 
If $r$ and $\sigma$ are small, the effects of entry-wise perturbations are relatively small so that all the algorithms give pretty similar results. 
However, if $r$ or $\sigma$ is large, the effects of entry-wise perturbations are pronounced in the kernel matrix, and therefore, the results are different depending on the robustness of the algorithms.
As shown in Figure \ref{fig:robustness}, the three $L_1$-norm kernel PCA algorithms and R1-KPCA outperform L2-KPCA as in the case of the linear kernel. 
However, while R1-KPCA achieves the best performance for the linear kernel, the $L_1$-norm based kernel PCA algorithms work better than R1-KPCA when the Gaussian kernel is used.
Especially, Algorithm \ref{alg:L1KernelPCA} outperforms all the other algorithms if $r$ exceeds $20$. The superior performance of Algorithm \ref{alg:L1KernelPCA} ranges from $1\%$ to $5\%$ in these cases.

\subsection{Outlier Detection}
\label{subsec:outlier_detection}
$L_2$-norm PCA has been shown to be effective for anomaly detection \cite{shyu2003novel}. The idea is to extract loading vectors using datasets consisting of only normal samples and use these loading vectors to develop a detection model. Specifically, a boundary of normal samples is constructed from the loading vectors and the boundary is used to discriminate normal and abnormal samples.

We extend this principle to outlier detection, i.e. its unsupervised counterpart. In the outlier detection setting, sample labels are not given when the model is built. Therefore, it is not possible to build a detection model solely based on normal samples. Given this context, we run robust kernel PCA algorithms on the entire dataset (with outliers) and use the resulting loading vectors to characterize a boundary of normal samples. Since these loading vectors are less influenced by outliers as illustrated in Section \ref{subsec:robust_extraction}, we expect that they would better construct a normal boundary. We compare the performance of Algorithm \ref{alg:L1KernelPCA} based models to that of KPCA-L1, L1-KPCA, R1-KPCA, and L2-KPCA based models as well as two other popular outlier detection models \cite{breunig2000lof} \cite{liu2008isolation}.

\subsubsection{Toy Examples}
We first illustrate the advantage of using robust kernel PCA for outlier detection using the following two-dimensional toy examples.
\begin{figure}[ht]
\centering
\includegraphics[trim={2cm 0 1cm 0},clip,scale=0.45]{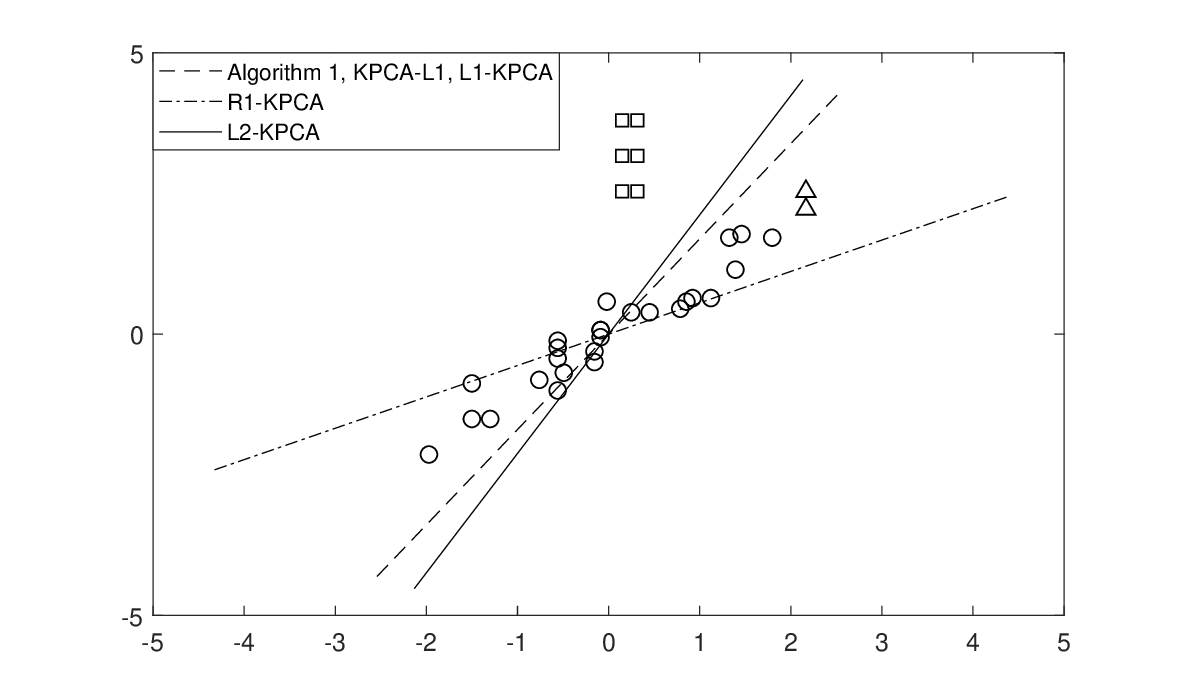}
\caption{The first toy example - original space}
\label{fig:toy_linear}
\end{figure}

Figure \ref{fig:toy_linear} displays the distribution of normal samples and outliers.
As the normal samples follow a linear pattern, we run the kernel PCA algorithms with the linear kernel and represent their first loading vectors in Figure \ref{fig:toy_linear}.
In the figure, the first loading vectors of the three $L_1$-norm based kernel PCA algorithms are represented using a single dashed line since they yield the same first loading vector in this example.
In addition to the normal samples forming a linear pattern, there are some outliers scattered exhibiting two different patterns; the two triangle points are outliers due to their scale and the six square points are outliers since they do not follow the linear pattern.
If the first loading vector exactly matches the linear pattern, outliers can be easily detected in the principal space; the triangle points can be detected due to large first principal components and the square points can be detected from large second principal components. 
However, due to the presence of outliers, it is impossible that the first loading vector exactly matches the linear pattern.
Given this context, we use robust kernel PCA algorithms to obtain the first loading vector with lower deviation from the linear pattern.

\begin{figure*}[h]
\centering
\includegraphics[trim={4cm 0 3cm 0},clip,scale=0.57]{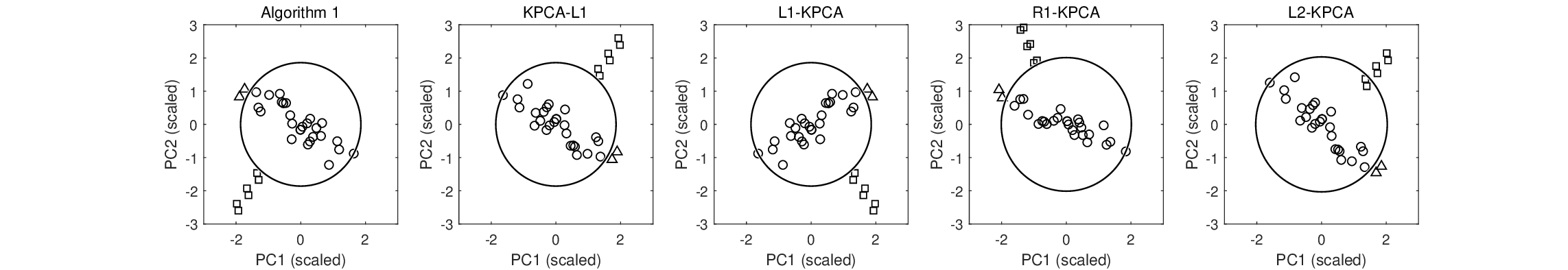}
\caption{The first toy example - principal space}
\label{fig:toy_PCA}
\end{figure*}

Figure \ref{fig:toy_PCA} displays the PCA results of the five kernel PCA algorithms.
In the figure, the x-axis and the y-axis represents the first and the second principal component, respectively.
As shown in the figure, the triangle outliers can be easily separated by the first principal component for any kernel PCA algorithm.
However, while the square outliers can be discriminated by the second principal component of the $L_1$-norm based kernel PCA algorithms and R1-KPCA, there exists some overlap between the normal samples and the square outliers in the range of the second principal component of L2-KPCA.
As seen in the figure, two outliers appear closer to the origin than some normal samples making the circular boundary of the normal samples include them.
On the other hand, all the normal samples are clearly separated from the outliers in the principal space of the $L_1$-norm based kernel PCA algorithms and R1-KPCA, demonstrating the advantage of using robust kernel PCA in outlier detection. This result is consistent with the findings in Figure \ref{fig:robustness_linear}.

\begin{figure}[ht]
\centering
\includegraphics[trim={1cm 0 1cm 0},clip,scale=0.55]{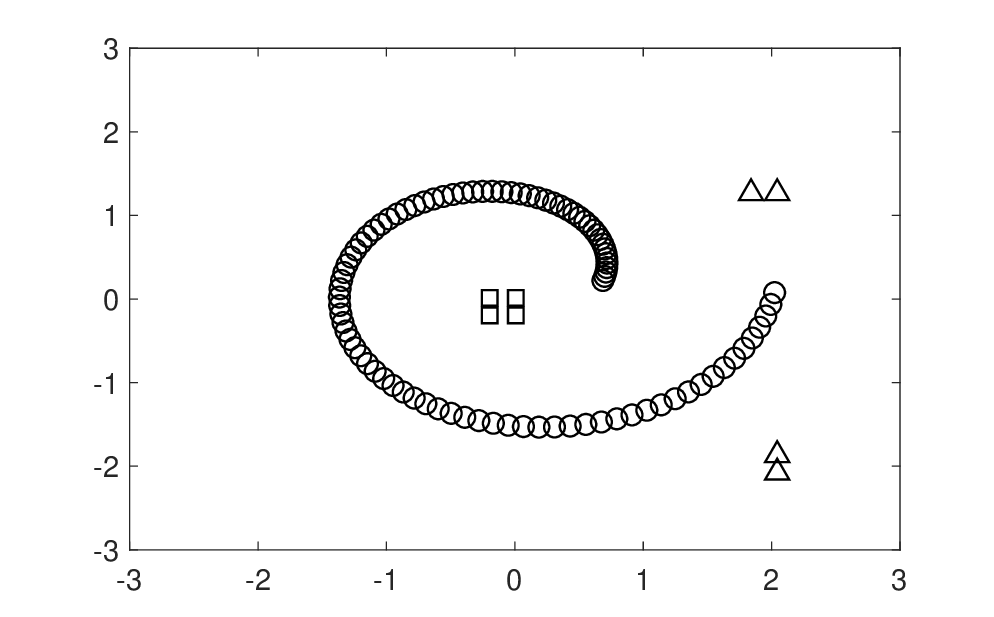}
\caption{The second toy example - original space}
\label{fig:toy_gaussian}
\end{figure}

In order to see if the same result holds for the Gaussian kernel, we consider another example. As shown in Figure \ref{fig:toy_gaussian}, the second example has a spiral pattern consisting of normal samples as well as two types of outliers.
As in the previous example, it has both trivial outliers (the triangle points) and more challenging outliers (the square points). 
In order to obtain nonlinear principal components, we run the five kernel PCA algorithms with the Gaussian kernel.
As Figure \ref{fig:toy_gaussian_PCA} displays, only Algorithm \ref{alg:L1KernelPCA} succeeds to exclude the square outliers from the boundary while the other kernel PCA algorithms include them within the boundary.
This superior performance of Algorithm \ref{alg:L1KernelPCA} with the Gaussian kernel is consistent with the results in Section \ref{subsec:robust_extraction} and attests the effectiveness of using it for outlier detection, especially with the Gaussian kernel.

\begin{figure*}[h]
\centering
\includegraphics[trim={4cm 0 3cm 0},clip,scale=0.55]{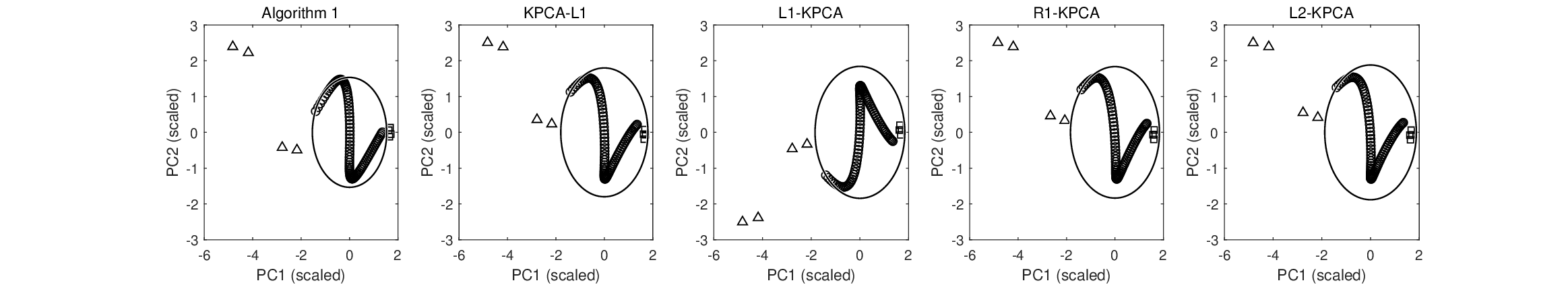}
\caption{The second toy example - principal space}
\label{fig:toy_gaussian_PCA}
\end{figure*}

\subsubsection{Real-world Datasets}

For outlier detection, we use datasets from the UCI Machine Learning Repository \cite{Lichman:2013} and the ODDS Library \cite{Rayana:2016}, see Table \ref{Outlier Detection - Dataset}.

\begin{table}[h]
\caption{Real-world Datasets for Outlier Detection}
\label{Outlier Detection - Dataset}
\vspace{-0.25cm}
\begin{center}
\begin{small}
\begin{tabular}{|l|r|r|r|}
\hline
Data set & \# samples & \# features & \# outliers \\
\hline
WBC    &	378	&	30	&	21 (7.6\%) \\
Ionosphere  & 351       & 33           & 126 (36\%) \\
BreastW	&	683	&	9	&	239	(35\%)	\\
Cardio  & 1831      & 21           & 176 (9.6\%) \\
Musk    & 3062 & 166 &	97 (3.2\%)\\
Mnist    & 	7603       &	100	& 700 (9.2\%) \\
\hline
\end{tabular}
\end{small}
\end{center}
\vskip -0.1in
\end{table}

In this experiment, we use a similar detection rule as the one in \cite{shyu2003novel} where it is applied for anomaly detection. Let $Y \in \mathbb{R}^{n \times p}$ denote $p$ principal components and $m_j$ and $\lambda_j$ be the mean and variance of the $j^{th}$ principal component, respectively. To detect outliers, we consider the following detection model, which classify the $i^{th}$ sample as an outlier if
\begin{align}
\sum\limits_{\{j:\lambda_j \geq \alpha\}} \frac{(Y_{ij}-m_j)^2}{\lambda_j} > c. \label{eq:detection_rule}
\end{align}

The metric appearing on the left-hand side of \eqref{eq:detection_rule} represents the squared Euclidean distance to the origin in the standardized principal space consisting of principal components whose variance is greater than or equal to $\alpha$. Therefore, our model can be understood as drawing a circular boundary (as illustrated in Figures \ref{fig:toy_PCA} and \ref{fig:toy_gaussian_PCA}) on this reduced standardized principal space. 
Since sample labels are unknown at the stage of building a model in the outlier detection setting, it is unclear how to choose an appropriate $c$.
So, we compute precision and recall with varying $c$ and evaluate the performance of each model using AUC under the precision-recall curve. We compare AUC of the Algorithm \ref{alg:L1KernelPCA} based models to that of the KPCA-L1, L1-KPCA, R1-KPCA, and L2-KPCA based models as well as that of the two popular outlier detection models, Local Outlier Factor (LOF) \cite{breunig2000lof} and Isolation Forest (iForest) \cite{liu2008isolation}.

Since principal components having small sample variance provide minor information, we only consider principal components whose sample variance is greater than or equal to some threshold value $\alpha$. We set $\alpha$ be to the largest $\bar{\alpha}$ such that 
\begin{center}
$ 0.8 \times \sum\limits_{j=1}^{d} \lambda_j \leq \sum\limits_{\{j:\lambda_j \geq \bar{\alpha} \}} \lambda_j$
\end{center}
holds where $d$ is the number of features. For the choice of the kernel function, we consider both the linear kernel and the Gaussian kernel with the width parameter $\sigma$ of the Gaussian kernel to be equal to $d$. On the other hand, we set the number of nearest neighbors to 10 in LOF, and the number of trees, the size of subsample, and the number of rounds to 100, 256, and 10, respectively in iForest since these parameter values are commonly used.

\begin{table*}[h]
\caption{AUC of the Outlier Detection Models}
\label{Outlier Detection - AUC}
\vspace{-0.25cm}
\begin{center}
\begin{threeparttable}
\begin{small}
\begin{tabular}{|c|c|c|c|c|c|c|c|c|c|c|c|c|}
\hline
\multirow{3}{*}{Datasets} & \multicolumn{12}{c|}{AUC}                                                                                                                                                                                                                                                                                                                                                                                                                                                                                                   \\ \cline{2-13} 
                          & \multicolumn{5}{c|}{Linear}                                                                                                                                                                                                         & \multicolumn{5}{c|}{Gaussian}                                                                                                                                                                                                       & \multirow{2}{*}{LOF} & \multirow{2}{*}{iForest} \\ \cline{2-11}
                          & Algo 1          & \begin{tabular}[c]{@{}c@{}}KPCA\\ -L1\end{tabular} & \begin{tabular}[c]{@{}c@{}}L1-\\ KPCA\end{tabular} & \begin{tabular}[c]{@{}c@{}}R1-\\ KPCA\end{tabular} & \begin{tabular}[c]{@{}c@{}}L2-\\ KPCA\end{tabular} & Algo 1          & \begin{tabular}[c]{@{}c@{}}KPCA\\ -L1\end{tabular} & \begin{tabular}[c]{@{}c@{}}L1-\\ KPCA\end{tabular} & \begin{tabular}[c]{@{}c@{}}R1-\\ KPCA\end{tabular} & \begin{tabular}[c]{@{}c@{}}L2-\\ KPCA\end{tabular} &                      &                          \\ \hline
WBC                       & 0.5208          & 0.5288                                             & 0.5320                                    & 0.4658                                             & 0.4798                                             & 0.5292          & \textbf{0.5340}                                    & \textbf{0.5337}                                    & 0.5072                                             & 0.5224                                             & 0.3451               & \textbf{0.5525}          \\ \hline
Ionosphere                & 0.6625          & \textbf{0.7319}                                    & 0.6834                                             & \textbf{0.7642}                                    & 0.7057                                             & \textbf{0.7238} & 0.6806                                             & 0.6887                                             & 0.7041                                             & 0.6992                                             & 0.7032               & 0.7067                   \\ \hline
Breastw                   & 0.9250          & 0.9125                                             & 0.9218                                             & 0.9269                                             & 0.9152                                             & \textbf{0.9428} & 0.9287                                             & {0.9354}                                             & \textbf{0.9521}                                    & {0.9309}                                             & 0.3750               & \textbf{0.9513}          \\ \hline
Cardio                    & {0.5790} & 0.5551                                             & \textbf{0.5799}                                    & 0.4265                                             & 0.5066                                             & \textbf{0.6096} & 0.5752                                    & \textbf{0.5963}                                    & 0.5116                                             & 0.4664                                             & 0.1921               & 0.5114                   \\ \hline
Musk                      & \textbf{0.9947}          & \textbf{0.9947}                                              & \textbf{0.9947}                                    & 0.8055                                             & 0.9358                                             & \textbf{0.9947}      & \textbf{0.9947}                                                     & \textbf{0.9947}                                    & 0.9916                                                   & \textbf{0.9947}                                    & 0.0925               & 0.7596                   \\ \hline
MNIST                     & \textbf{0.3985}                   &  \textbf{0.4002}                                                  &  N/A                                                  & N/A                                                &   {0.3914}                                                 & \textbf{0.3966}                 &  {0.3913}                                                  &     N/A                                               & N/A                                                   &    0.3639                                                & 0.1924               & 0.3380                   \\ \hline
\end{tabular}
    \begin{tablenotes}
      \small
      \item[*] N/A: The experiments can not be completed within the period of 24 hours.
    \end{tablenotes}
\end{small}
\end{threeparttable}
\end{center}
\vskip -0.1in
\end{table*}

Table \ref{Outlier Detection - AUC} displays the AUCs of the 12 different detection models. 
The numbers in bold present the highest AUC cases (there can be several similar top performances).
If outliers are obvious, any kernel PCA based model works well as seen in the case of Breastw and Musk. 
However, if outliers are unclear, the Algorithm \ref{alg:L1KernelPCA} based detection models tend to outperform the other detection models.
Especially, the Algorithm \ref{alg:L1KernelPCA} based model with the Gaussian kernel consistently achieves top AUC values.
Compared to the kernel PCA based models, LOF and iForest do not work well.
LOF never achieves the top performance and iForest is not competitive for high-dimensional datasets such Must and MNIST although it yields the top AUC values for WBC and Breastw.
As opposed to them, the Algorithm \ref{alg:L1KernelPCA} based model with the Gaussian kernel consistently works well regardless of the size of the problem, demonstrating its effectiveness in outlier detection.

\subsection{Runtime Comparison}
\label{subsub:runtime_comparison}
Lastly, we compare the runtime of Algorithm \ref{alg:L1KernelPCA} to that of KPCA-L1, L1-KPCA, R1-KPCA, and L2-KPCA.
In order to obtain a runtime comparison, we run them on the six real-world datasets presented in Table \ref{Outlier Detection - Dataset} and measure the time taken to get all the principal components.

\begin{table*}[h]
\caption{Runtime Comparison}
\label{Outlier Detection - Runtime}
\vspace{-0.25cm}
\begin{center}
\begin{small}
\begin{tabular}{|c|r|r|r|r|r|r|r|r|r|r|}
\hline
\multirow{3}{*}{Datasets} & \multicolumn{10}{c|}{Runtime (minutes)}                                                                                                                                                                                                                                                                                                                                                                                                                                                                                                                                                                                                                                   \\ \cline{2-11} 
                          & \multicolumn{5}{c|}{Linear}                                                                                                                                                                                                                                                                                                         & \multicolumn{5}{c|}{Gaussian}                                                                                                                                                                                                                                                                                                       \\ \cline{2-11} 
                          & \multicolumn{1}{c|}{Algo 1} & \multicolumn{1}{c|}{\begin{tabular}[c]{@{}c@{}}KPCA\\ -L1\end{tabular}} & \multicolumn{1}{c|}{\begin{tabular}[c]{@{}c@{}}L1-\\ KPCA\end{tabular}} & \multicolumn{1}{c|}{\begin{tabular}[c]{@{}c@{}}R1-\\ KPCA\end{tabular}} & \multicolumn{1}{c|}{\begin{tabular}[c]{@{}c@{}}L2-\\ KPCA\end{tabular}} & \multicolumn{1}{c|}{Algo 1} & \multicolumn{1}{c|}{\begin{tabular}[c]{@{}c@{}}KPCA\\ -L1\end{tabular}} & \multicolumn{1}{c|}{\begin{tabular}[c]{@{}c@{}}L1-\\ KPCA\end{tabular}} & \multicolumn{1}{c|}{\begin{tabular}[c]{@{}c@{}}R1-\\ KPCA\end{tabular}} & \multicolumn{1}{c|}{\begin{tabular}[c]{@{}c@{}}L2-\\ KPCA\end{tabular}} \\ \hline
WBC                       & 0.0                        & 0.0                                                                    & 0.1                                                                    & 0.2                                                                   & 0.0                                                                    & 0.0                        & 0.0                                                                    & 0.0                                                                   & 0.2                                                                    & 0.0                                                                   \\ \hline
Ionosphere                & 0.0                        & 0.0                                                                    & 0.1                                                                & 0.2                                                                   & 0.0                                                                    & 0.0                        & 0.0                                                                    & 0.1                                                                    & 0.2                                                                    & 0.0                                                                    \\ \hline
Breastw                   & 0.0                        & 0.0                                                                    & 0.1                                                                    & 0.2                                                                   & 0.0                                                                    & 0.0                        & 0.0                                                                    & 0.1                                                                   & 0.1                                                                    & 0.0                                                                    \\ \hline
Cardio                    & 0.0                        & 1.9                                                                  & 12.3                                                                 & 7.6                                                                  & 0.1                                                                   & 0.0                        & 1.8                                                                   & 15.3                                                                 & 6.2                                                                  & 0.1                                                                    \\ \hline
Musk                      & 0.5                       & 69.6                                                                 & 999.1                                                              & 973.5                                                               & 0.3                                                                    & 0.5                       & 12.9                                                                 & 1018.0                                                              & 968.3                                                               & 0.1                                                                   \\ \hline
MNIST                     & 0.9				        & 263.3                                             & $>1440$
& $>1440$
& 1.8                                                  & 1.0				        & 263.2                                              & $>1440$                                                     & $>1440$                                                  & 1.8                                                 
\\ \hline
\end{tabular}
\end{small}
\end{center}
\vskip -0.1in
\end{table*}

As shown in Table \ref{Outlier Detection - Runtime}, the runtime largely varies across the algorithms. 
Among the $L_1$-norm based kernel PCA algorithms, Algorithm \ref{alg:L1KernelPCA} has the smallest runtime for all datasets.
Actually, it is much faster than the other two algorithms since it requires only one matrix-vector multiplication while the other algorithms entail either eigen-decomposition or solving a system of equations.
R1-KPCA is also not as fast as Algorithm \ref{alg:L1KernelPCA} since it involves QR-decomposition in each iteration to make loading vectors orthogonal.
Among the robust kernel PCA algorithms, only Algorithm \ref{alg:L1KernelPCA} is computationally comparable to L2-KPCA, making it the best choice for robust kernel PCA in a large-scale setting.

\section{Conclusion}
In this work, we present a simple algorithm for $L_1$-norm kernel PCA and provide its convergence analysis. In order to develop it, we first reformulate $L_1$-norm kernel PCA into a geometrically interpretable problem and derive a geometric interpretation behind it. Based on the geometric interpretation, we develop an algorithm to which the kernel trick is applicable. In the convergence analysis, we prove that the algorithm converges to a local optimal solution in a finite number of steps and the sequence of objective values converges at a linear rate. 

The computational experiments demonstrate the robustness of the proposed algorithm in the presence of entry-wise perturbations and the runtime comparison shows that it takes much less time than the other robust kernel PCA algorithms. Also, its application to outlier detection outperforms all of the other benchmark algorithms. The model based on the proposed algorithm is not only better than that of the other kernel PCA based models but also outperforms LOF and iForest, especially when high-dimensional datasets are considered.


%



%

\bibliographystyle{IEEEtran}
\bibliography{main}

%

\begin{IEEEbiography}[{\includegraphics[width=1in,height=1.25in,clip,keepaspectratio]{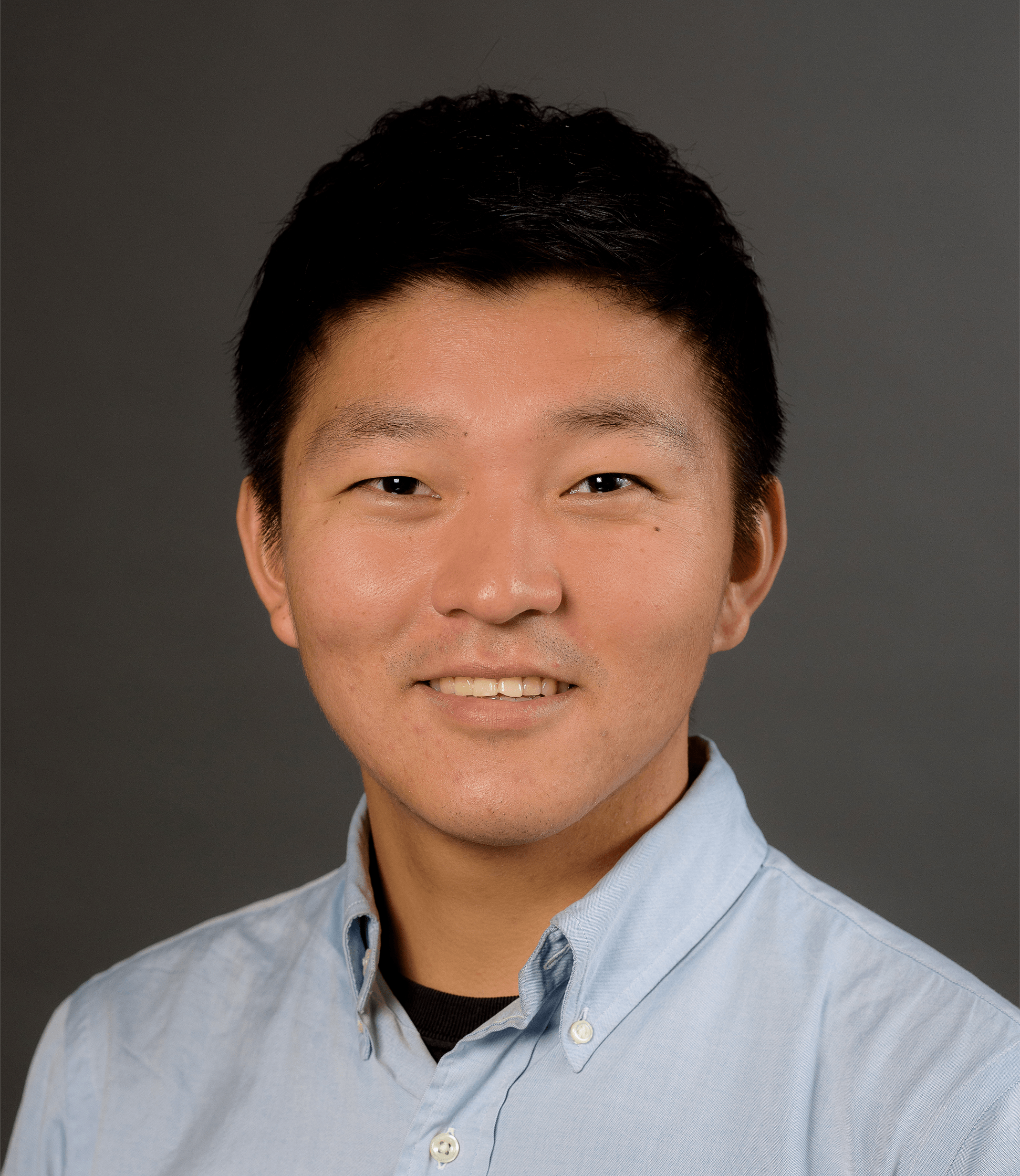}}]{Cheolmin Kim}
is a Ph.D student in Industrial Engineering and Management Science at Northwestern University. He received his B.S in Industrial Engineering and B.A in Economics from Seoul National University. His research interests are at the interaction of optimization and machine learning. He is interested in designing a new machine learning model, developing an training algorithm and analyzing the efficiency of the algorithm from an optimization perspective.
\end{IEEEbiography}

\begin{IEEEbiography}[{\includegraphics[width=1in,height=1.25in,clip,keepaspectratio]{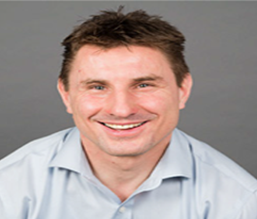}}]
{Diego Klabjan} is a professor at Northwestern University, Department of Industrial Engineering and Management Sciences. He is also Founding Director, Master of Science in Analytics. After obtaining his doctorate from the School of Industrial and Systems Engineering of the Georgia Institute of Technology in 1999 in Algorithms, Combinatorics, and Optimization, in the same year he joined the University of Illinois at Urbana-Champaign. In 2007 he became an associate professor at Northwestern and in 2012 he was promoted to a full professor. His research is focused on machine learning, deep learning and analytics with concentration in finance, transportation, sport, and bioinformatics. Professor Klabjan has led projects with large companies such as Intel, Baxter, Allstate, AbbVie, FedEx Express, General Motors, United Continental, and many others, and he is also assisting numerous start-ups with their analytics needs. He is also a founder of Opex Analytics LLC.
\end{IEEEbiography}





\end{document}